\PassOptionsToPackage{numbers,compress}{natbib}
\documentclass{article}

\usepackage[preprint]{neurips_2025}

\usepackage[utf8]{inputenc} % allow utf-8 input
\usepackage[T1]{fontenc}    % use 8-bit T1 fonts
\usepackage{hyperref}       % hyperlinks
\usepackage{url}            % simple URL typesetting
\usepackage{booktabs}       % professional-quality tables
\usepackage{amsfonts}       % blackboard math symbols
\usepackage{nicefrac}       % compact symbols for 1/2, etc.
\usepackage{microtype}      % microtypography
\usepackage{xcolor}         % colors

\usepackage{graphicx}
\usepackage{amsmath}
\usepackage{subcaption}

\usepackage{amsmath}
\usepackage{amssymb}
\usepackage{mathtools}
\usepackage{amsthm}
\newcommand{\equalcontrib}{\thanks{Equal contribution}}
\newcommand{\equalcontribmark}{\footnotemark[\value{footnote}]}

\usepackage[capitalize,noabbrev]{cleveref}

\theoremstyle{plain}
\newtheorem{theorem}{Theorem}[section]

\newtheorem{lemma}[theorem]{Lemma}

\theoremstyle{definition}
\newtheorem{definition}[theorem]{Definition}

\theoremstyle{remark}

\usepackage{float}

\usepackage[textsize=tiny]{todonotes}
\usepackage{multirow}

\title{SALMAN: \underline{S}tability \underline{A}nalysis of  \underline{L}anguage Models\\Through the Maps Between Graph-based \underline{Man}ifolds}

\author{%
  Wuxinlin Cheng\equalcontrib \\
  Stevens Institute of Technology\\
  \texttt{chengwu1230@gmail.com} \\
  \And
  Yupeng Cao\equalcontribmark \\
  Stevens Institute of Technology\\
  \texttt{ycao33@stevens.edu} \\
  \And
  Jinwen Wu\\
  Stevens Institute of Technology\\
  \texttt{jwu74@stevens.edu} \\
  \And
  Koduvayur Subbalakshmi \\
  Stevens Institute of Technology\\
  \texttt{ksubbala@stevens.edu} \\
  \And
  Tian Han \\
  Stevens Institute of Technology\\
  \texttt{than6@stevens.edu} \\
  \And
  Zhuo Feng\\
  Stevens Institute of Technology\\
  \texttt{zfeng12@stevens.edu} \\
}

\begin{document}

\maketitle

\begin{abstract}
Recent strides in Pretrained Transformer-based language models have propelled state-of-the-art performance in numerous NLP tasks. Yet, as these models grow in size and deployment, their robustness under input perturbations becomes an increasingly urgent question. Existing robustness methods often diverge between small-parameter and large-scale models (LLMs), and they typically rely on labor-intensive, sample-specific adversarial designs. In this paper, we propose a unified, local (sample-level) robustness framework (SALMAN) that evaluates model stability without modifying internal parameters or resorting to complex perturbation heuristics. Central to our approach is a novel Distance Mapping Distortion (DMD) 
measure, which ranks each sample’s susceptibility by comparing input-to-output distance mappings in a near-linear complexity manner. By demonstrating significant gains in attack efficiency and robust training, we position our framework as a practical, model-agnostic tool for advancing the reliability of transformer-based NLP systems.
\end{abstract}

\section{Introduction}
%\than{maybe could also shrink a bit about the introduction, especifially the first few paragraphs that describe other methods. The section 2 Background could also be shrinked, especially the sec2.1. As we have lots of results/theorems/tables to put into the paper, those sentences/paragraphs that describe those methods can be relatively reduced. They are important to have as they can motivates our method, but maybe less important compared to our major innovation/methodology }
Recently, breakthroughs in pretrained Transformer-based language models have revolutionized the field of Natural Language Processing (NLP). These models have enhanced performance across a wide range of downstream NLP tasks, including text classification~\cite{sun2019fine}, summarization~\cite{el2021automatic}, chatbot~\cite{achiam2023gpt}, and complex reasoning~\cite{wei2022chain}. Given the widespread adoption of language models, it is crucial to evaluate their robustness. The robustness problems in the NLP community mostly focus on exploring the language model behavior when the inputs are modified.

%Initially, 
\citet{ebrahimi2017hotflip} design character-level and word-level perturbations as adversarial examples to attack NLP models. \citet{jia2017adversarial} explore the method %\suba {methods}\yc{done} 
to mislead the language model's output in the Q\&A task by adding random sentences. Later, research such as that by ~\citet{jin2020bert} and~\citet{li2020bert} focused on designing adversarial samples that better preserve the original semantics. Subsequently, some work began to analyze the robustness of language models in continuous space and improve the generalization ability of NLP models and defense against word substitution attacks through adversarial training in continuous space~\cite{zhufreelb, li2021searching, li2023perturbscore}. Recently, the growth of model parameter size and training data for language models has demonstrated that Large Language Models (LLMs) exhibit increased robustness to trivial disturbances and handle common disruptions more effectively~\cite{achiam2023gpt, zou2023universal}. As a result, recent studies have devised Jailbreak prompts specifically designed to attack LLMs, thereby evaluating and testing their robustness~\cite{wang2023decodingtrust, zhu2023promptrobust}.

While significant progress has been made in the robust evaluation of pre-trained language models, current robustness analyses still face several key challenges. First, the methods for evaluating robustness in large language models (e.g., LLaMA-series) and those in smaller models (e.g., BERT, BART) differ substantially. In smaller models, word-level or token-level adversarial attacks often suffice to expose vulnerabilities~\cite{li2018textbugger, li2020bert, garg2020bae}. However, large language models can often interpret or adapt to these simple perturbations, necessitating more carefully designed prompt-based strategies for effective robustness testing~\cite{zou2023universal, paulus2024advprompter}. As a result, a unified robustness evaluation framework applicable to both large and small models is currently lacking. Second, irrespective of a model’s parameter size, designing adversarial inputs or prompts remains a time-consuming and labor-intensive process, particularly for large-scale NLP datasets. These challenges highlight the need for automated, scalable, and more universal approaches to evaluate the robustness of diverse language models. %\yc{Refined this paragraph.}\suba{It would be worthwhile pointing out why you think having one method to analyze the robustness of LMs of different sizes is important. Presumably, the smaller models will be preferred if they can do the job and be robust, so it is ok if they are evaluated in a different way, as long as you trust that methodology?} 

In this work, we propose a sample-centric robustness framework that addresses both challenges by: 1) Unifying Robustness Evaluation for All Model Scales. 
Specifically, our method computes a local (per-sample) robustness measure—applicable to both smaller models (e.g., DistilBERT) and massive LLMs (e.g., GPT-2, Llama) \emph{without} requiring changes to their internal parameters or specialized adversarial training. 2) Minimizing Labor-Intensive Perturbation Design. Instead of heavily relying on constructing adversarial prompts, we quantify each input’s inherent vulnerability via a lightweight, near-linear complexity approach. This ranking then guides adversarial attacks or fine-tuning decisions, drastically reducing manual effort compared to purely sample-by-sample adversarial generation.

%\than{maybe also mention PGM and briefly say how you use it and its high level advantages. }\wc{maybe it's better to mention PGM after the background introduction?}\than{One concern is section 2.2 discuss the PGM, but we haven't talked about this PGM before this section. I think it might be better to briefly describe the PGM in the introduction because our title explicitly mention graph-based manifold, this graph is related to PGM, and we need such graph to define DMD. Even if we dont discuss PGM in the intro, then we might breifly say thing to warm up the idea, then it will become smooth to take about PGM in the background. }\wc{sure, todo}
At the core of our method is a novel, per-sample distance mapping distortion (DMD) metric that compares distances in the input representation space against distances in the output representation space. To facilitate these distance calculations efficiently, we first build a \emph{near-linear complexity} Probabilistic Graphical Model (PGM) that captures the manifold structure of the data, preserving both local geometry and global structural properties without resorting to dense or iterative global optimizations. 
%\suba{maybe give a high level understanding on why you think this is a valid way to measure robustness of samples. And tie it in to how this relates to robustness analysis for LLMs}\wc{done}
By assessing each instance individually, we gain a fine-grained view of where and how a model fails to preserve distances across its representations—leading directly to broader insights about the system’s behavior as a whole. Such per-sample analyses, in turn, form the building blocks of understanding overall stability~\cite{zhang2019theoretically}. Rather than relying on aggregate statistics alone, examining each sample’s distortion enables us to pinpoint particular modes of fragility.
We show how this ranking can: 1) Streamline NLP Attacks: Targeting non-robust samples first yields more efficient and more effective adversarial attacks (Section~\ref{subsec:guided-attack}). 2) Improve LLM robustness through Fine-Tuning: Up-weighting non-robust data during fine-tuning preserves or even improves generalization and yields internal representations closer to the pre-trained checkpoint (Section~\ref{subsec:guided-finetuning}). %\than{how does the descriptions above related to fine-tuning stability}.\wc{done} 
Furthermore, the same method applies to both smaller-scale models and large-scale LLMs, offering a unified pathway for robustness analysis across diverse parameter regimes.

Overall, our contributions are:
\begin{itemize}
    \item A unified robustness measure (SALMAN) that can be computed in nearly-linear time for language models of varied sizes (from smaller transformers to LLMs), without requiring specialized tasks, perturbed data, or parameter modifications.
    \item To our best knowledge, SALMAN is the first local (sample-level) robustness measurement specifically tailored from small to large language models%\than{maybe slightly revise the narrative. The first contribution say the measure works for both small and LLM, but here, we say the metric is specifically tailored to LLM}
    , enabling fine-grained analysis of how individual inputs withstand minor or adversarial perturbations. %\than{maybe add to our best knowledge somewhere in case that reviewer find some early work}
    %\item A scalable, near-linear procedure to compute per-sample vulnerability scores, minimizing the labor-intensive process of designing or generating adversarial samples for each input.\than{this point might be slightly redundant with the second point? Both point seem emphasize that we propose a sample-specific measurement}
    \item Empirical demonstrations across both small (BERT, DistilBERT) and large models (GPT-2, Llama) showing how this sample-level perspective leads to (i) more efficient and higher success-rate adversarial attacks, and (ii) improved robust fine-tuning outcomes.
\end{itemize}

\section{Background}

\subsection{Robustness in NLP}
\label{subsec:rb}
The robustness of language models remains a pivotal area of research within the NLP community. Several studies have explored the vulnerability of these models to modifications in the input text, ranging from typos to word replacements~\cite{li2020bert, jin2020bert, sun2020adv}. \citet{wang2021adversarial} further developed a multi-task benchmark to evaluate language model robustness. In the realm of model probing, \citet{tenney2019bert} and \citet{hewitt2019structural} examined how syntactic and semantic features are represented across different layers of BERT-like models. \citet{voita2019bottom} and \citet{abnar2019blackbox} employed similarity-based analysis methods to study the evolution of representations in deep neural networks. \citet{zhou2021closer} and \citet{neerudu2023robustness} performed a comprehensive analysis of how finetuning affects the representations in the language model using a combination of probing and analytical techniques. With the increase in model parameters, LLMs can distinguish between minor textual variations, underlining the need to explore their robustness to input perturbations.  Recent studies have focused on the impact of input prompts on LLM robustness~\cite{shayegani2023survey}. ~\citet{wang2023robustness} assessed the robustness of ChatGPT against adversarial and out-of-distribution samples. \citet{zou2023universal} enhanced the efficiency of jailbreak attacks by generating adversarial suffixes. DecodingTrust examined the robustness of LLMs using standard datasets like AdvGLUE and AdvGLUE++~\cite{wang2023decodingtrust}. PromptRobust investigated the robustness of LLMs from the perspective of prompts, demonstrating that subtle changes in instructions can lead to significant performance variations~\cite{zhu2023promptrobust}.

\subsection{Probabilistic Graphical Models}
\label{subsec:pgms}

%Probabilistic Graphical Models (PGMs) provide a powerful framework for representing conditional dependencies among a set of variables in a graph structure \cite{koller2009probabilistic}. In our context, each \emph{sample} (e.g., an embedding from a Transformer) corresponds to a node in the graph, with edges encoding meaningful interactions or similarities between those samples. By encapsulating joint distributions over high-dimensional data into graph-based factorizations, PGMs allow for both interpretability and efficient inference. Building a PGM also captures both local and global dependencies among samples~\cite{vu2020pgm, feng2021sgl}. Specifically, nodes that share strong proximity or structural similarity form edges that collectively approximate the intrinsic manifold of the data. Each node maintains a local neighborhood capturing conditional dependencies, and edges approximate the manifold’s connectivity structure~\cite{rubin2020manifold}. Consequently, we obtain a \emph{probabilistic} view of the data manifold: points forming tighter subgraphs indicate higher intrinsic similarity, while loosely connected regions indicate greater divergence. Recently, SAGMAN~\cite{cheng2024sagman} extend these PGM-based ideas to graph neural networks by combining dimension reduction to handle domain-specific manifold structures.

Probabilistic Graphical Models (PGMs) represent conditional dependencies among variables in a graph, enabling interpretability and efficient inference~\cite{koller2009probabilistic}. Here, each sample (e.g., a Transformer embedding) becomes a node, with edges capturing local/global interactions that approximate the data manifold~\cite{vu2020pgm, feng2021sgl, rubin2020manifold}. Tightly connected subgraphs indicate higher intrinsic similarity, while loosely connected regions suggest divergence. Recently, SAGMAN~\cite{cheng2024sagman} extends these PGM-based ideas to GNNs by incorporating dimension reduction for domain-specific manifold structures.

\section{Theoretical Foundations of SALMAN}

\begin{figure*}[ht]
    \centering
    \includegraphics[width=0.7\textwidth, trim={2.5cm 1.0cm 4.5cm 2.4cm}, clip]{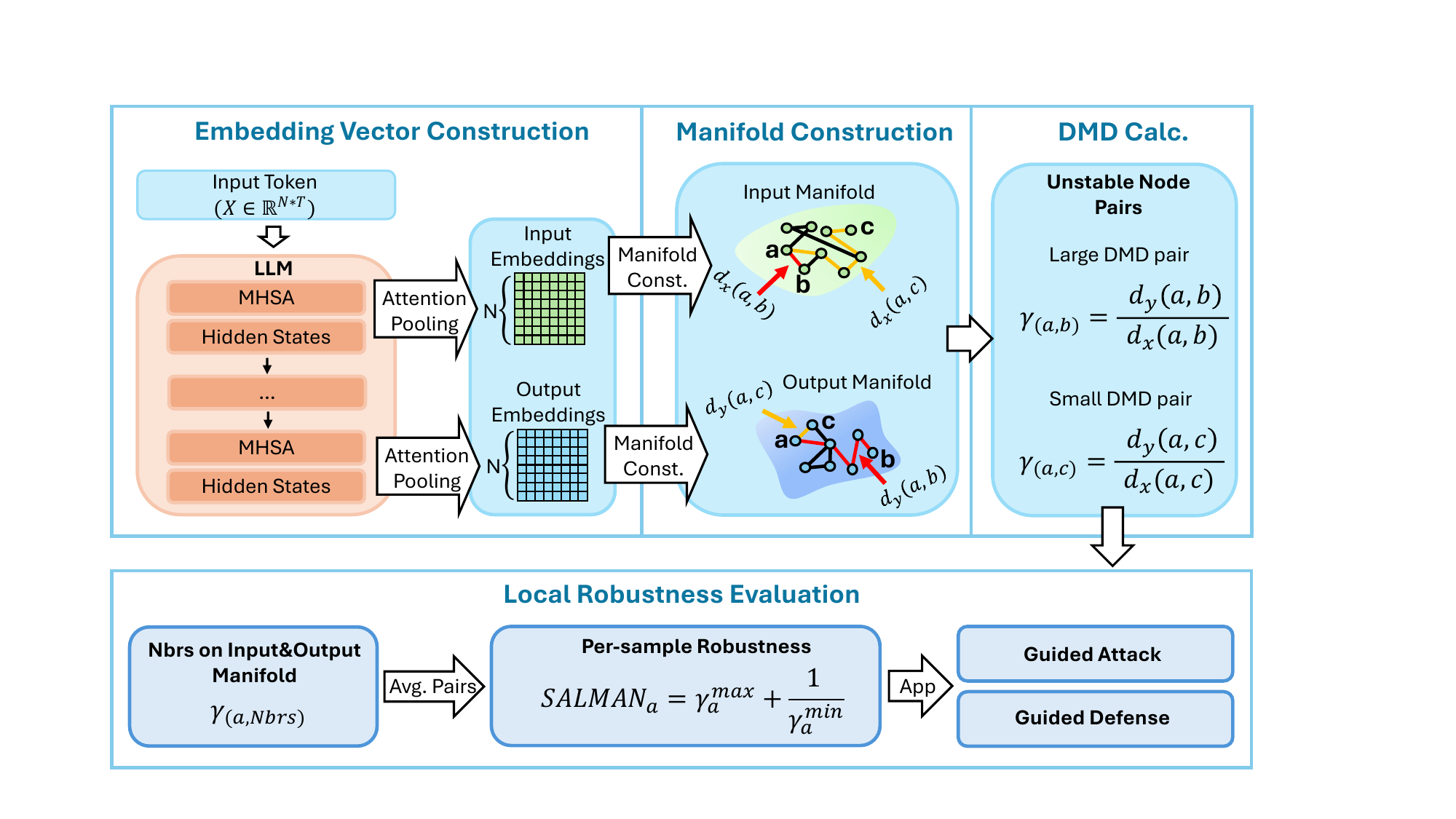}
    \vspace{-5pt}
    \caption{The overview of SALMAN Method.}
    % \vspace{-15pt}
\label{fig:flowchart}
\end{figure*}

In this section, we detail our overall pipeline (Figure~\ref{fig:flowchart}), followed by the mathematical underpinnings of (1) Embedding Construction (Sec.~\ref{subsec:dim-reduction}), (2) Manifold Construction (Sec.~\ref{subsec:manifold_construction}), (3) Distance Mapping Distortion (Sec.~\ref{subsec:dmd-calculation}), and (4) Algorithm Complexity Analysis (Sec.~\ref{subsec:algo-complexity}).

\subsection{Embedding Vector Construction}
\label{subsec:dim-reduction}
A key challenge in analyzing transformer models arises from the \emph{discrete} nature of token embeddings, which may not form well-behaved manifolds in the topological sense~\citep{robinson2024structure}. In particular, the geometry of the token space is heavily fragmented: a small textual perturbation (e.g., replacing a token with a synonym) can induce a disproportionately large jump in token-level embedding indices. As a result, continuous manifold-based analyses---which rely on smooth neighborhoods and gradual changes---become intractable when applied directly to discrete tokens. 
%\paragraph{Motivation for Continuous Embeddings.}
%Since many modern NLP models yield high-dimensional token embeddings~\cite{kenton2019bert,xu2023tensorgpt}, directly treating each token as a separate data point can further inflate dimensionality, complicating subsequent manifold or stability analyses. 
Moreover, transformers often exhibit \emph{stochastic decoding} (via temperature sampling, beam search, etc.)~\cite{li2024dynamic}, meaning identical input text can produce slightly different token outputs. Hence, relying purely on discrete token sequences introduces variability that disrupts stable manifold construction.

%each data sample $\mathbf{x}_i$ is tokenized embeddings
%$\{\mathbf{x}_{i,1}, \mathbf{x}_{i,2}, \dots, \mathbf{x}_{i,T_i}\}$ of length $T_i$. 
\paragraph{Attention Based Embedding Representation.}
To address these issues, we aggregate each sample’s token embeddings into a \emph{single} continuous vector, thereby avoiding the discontinuities of the raw token space. Formally, given a natural language dataset of $N$ samples, each data sample is tokenized into a sequence of embeddings $\{\mathbf{x}_1,\dots,\mathbf{x}_N\}$. We then pass each $\mathbf{x}_i$ through a Transformer-based pre-trained language model to obtain its Multi-Head Self-Attention (MHSA) \emph{outputs}, which we denote as $A_i 
  \;=\; \mathrm{MHSA}\bigl(\mathbf{x}_i\bigr) 
  \;\in\; \mathbb{R}^{H \times T_i \times d_{\mathrm{model}}}.$
%\[
%  A_i 
%  \;=\; \mathrm{MHSA}\bigl(\mathbf{x}_i\bigr) 
%  \;\in\; \mathbb{R}^{H \times T_i \times d_{\mathrm{model}}}.
%\]
Here, $H$ is the number of attention heads, $T_{i}$ is length of $\mathbf{x}_i$, and $d_{\mathrm{model}}$ is the hidden dimension. Then, we average these per-head outputs along the head dimension: $ \overline{A}_i 
  \;=\; \frac{1}{H} \sum_{h=1}^H A_{i,h}
  \;\in\; \mathbb{R}^{T_i \times d_{\mathrm{model}}}$
%\[
%  \overline{A}_i 
%  \;=\; \frac{1}{H} \sum_{h=1}^H A_{i,h}
%  \;\in\; \mathbb{R}^{T_i \times d_{\mathrm{model}}},
%\].
. Next, we \emph{compute an attention-based weighting}. The Softmax is applied to attention matrices $\overline{A}_i$ over the $T_i$ tokens to obtain $\boldsymbol{\alpha}_i\in\mathbb{R}^{T_i}$. %\yc{ Refined.}\than{where does this $\alpha$ come from? How do we compute it? Do we softmax on $\overline{A}_i$? }
%\[
%  \sum_{t=1}^{T_i} \boldsymbol{\alpha}_{i,t} \;=\; 1.
%\]
%Let $\boldsymbol{\alpha}_i\in\mathbb{R}^{T_i}$ denote the chosen 
%\emph{nonnegative} weights for $i-th$ token in input sequence, which 
%We normalize $\sum_{t=1}^{T_i} \boldsymbol{\alpha}_{i,t} \;=\; 1.$
Finally, we compute a single pooled vector 
$\mathbf{v}_i\in\mathbb{R}^{d_{\mathrm{model}}}$ for each sequence by a 
weighted sum of the token embeddings: $\mathbf{v}_i 
  \;=\; 
  \sum_{t=1}^{T_i} \boldsymbol{\alpha}_{i,t} 
  \;\overline{A}_{i,t}.$
%\[
%  \mathbf{v}_i 
%  \;=\; 
%  \sum_{t=1}^{T_i} \boldsymbol{\alpha}_{i,t} 
%  \;\overline{A}_{i,t}.
%\]
Collecting all $\mathbf{v}_i$ for $i=1,\dots,N$ then yields a new embedding matrix $\{\mathbf{v}_1,\dots,\mathbf{v}_N\}$. We denote the embedding matrix taken from the first layer in the language model as $\mathbf{z}_X$ and the embedding matrix taken from the last layer of the language model as $\mathbf{z}_Y$. 
%Following the same process, a new embedding output matrix $\mathbf{z}_Y$ will be yielded. 
%\than{what is $y$?}\wc{should We mention the MHSAs of input and output are from different LLM layers?}

% \textbf{Stop At Here.}
% \begin{equation}
%     \mathbf{z}_X \;=\; \mathrm{Pool}\bigl(\mathbf{x}_1,\dots,\mathbf{x}_n\bigr),
%     \label{eq:pool_input}
% \end{equation}
% where $\mathrm{Pool}(\cdot)$ is the output of Multi-Head Self-Attention (MHSA) from LLM. This yields a linguistically meaningful input embedding $\mathbf{z}_X$~\cite{clark2019does} . Likewise, we pool the final-layer hidden states to obtain a single \emph{output embedding} $\mathbf{z}_Y$:
% \begin{equation}
%     \mathbf{z}_Y \;=\; \mathrm{Pool}\bigl(\mathbf{h}_1,\dots,\mathbf{h}_m\bigr),
%     \label{eq:pool_output}
% \end{equation}
% where $\{\mathbf{h}_1,\dots,\mathbf{h}_m\}$ are the final hidden states prior to the softmax. %In this work, we found that \textbf{(i)} average-pooling and \textbf{(ii)} mean-plus-attention pooling each preserve critical semantic information while reducing dimensionality.

\paragraph{Deterministic Hidden-State Embeddings.}
Though transformers can produce stochastic token outputs, the internal hidden states remain largely deterministic \emph{if we freeze the model parameters and inference procedure}. For instance, by disabling dropout layers and using a fixed random seed, we empirically found that $\mathbf{z}_Y$ becomes stable irrespective of minor token-level variations. %In our preliminary experiments%\than{better change the words. Preliminary experiments sounds like the work is incomplete and not solid yet. Maybe start with specifically or to be more precise}, 
Specifically, we measured multiple $\mathbf{z}_Y$ across the same input and observed a significant similarity improvement %\than{improvement in what?} 
compared token embeddings (As detailed in the Appendix~\ref{sec:appendix_deterministic}). %This indicates that even under mild stochastic decoding, the pooled embeddings stay consistent enough for our manifold-based analyses.

% \paragraph{Connection to Manifold Construction.}
% We employ the continuous embeddings $\{\mathbf{z}_X, \mathbf{z}_Y\}$ to build manifolds in the subsequent phases of our pipeline. Specifically:
% \begin{itemize}
%     \item \textbf{Section~\ref{sec:manifold}} constructs a graph-based manifold over these embedding vectors, preserving \emph{effective resistance} distances between samples.
%     \item \textbf{Section~\ref{sec:DMD}} defines a \emph{Distance Mapping Distortion} metric on these manifolds, quantifying how robustly the model maps small perturbations in $\mathbf{z}_X$ to changes in $\mathbf{z}_Y$.
% \end{itemize}
% Thus, the final manifold analysis—critical for our stability and adversarial robustness studies—rests on having a stable, continuous representation of each sequence.

By aggregating discrete token embeddings into a single high-level embedding, we circumvent the discontinuities of token-level spaces. We thereby ensure that (1) manifold analysis is tractable (since $\mathbf{z}_X, \mathbf{z}_Y$ both lie in continuous $\mathbb{R}^d$~\cite{mehta2019define}), and (2) stochastic decoding does not cause major geometric shifts in these embeddings. This design choice, while straightforward, underpins all subsequent sections on manifold construction and robustness evaluation.

\subsection{Construction of Manifolds via PGM}
\label{subsec:manifold_construction}
%\than{this section might be benefit from re-organizing or revision. The logic is not quite clear (especially the motivation and presentation) to me when I first read it. a).If we want to present the whole procedure like in fig1, we should present the graph construction first, then mention its limitation, and then naturally motivates the edge sparsification. But its unclear how to construct the initial graph (maybe using k-nn?) the paragraph directly jump in to describe the edge pruning techniques. b) its unclear to me why we begin with the Eqn 1. If Eqn 1 is only linked to the sparsification technique, why not we present the pruning first, and then states that such prunning theorectically could maximize the F($\Theta$), then provide some insights/explaination about this $F(\Theta)$. For example, because such sparsification max the F($\Theta$), so the xx term will be maximized, so the pruning may achieve xxx desired properties etc. If we present the Eqn 1 in the current format, then people would expect you explicitly describe some update procedure of learning $\Theta$, even if you link the Eqn 1 to edge prunning, the audience maybe curious why we need to solve (1).}\wc{done}

%\wc{todo: logical flow: need manifold for stability analysis, direct constrict manifold with high dimension data can be hard, so we use PGMs}

Understanding a language model’s local robustness involves assessing how small input perturbations influence the model’s output representation. A common strategy is to interpret this input--output mapping as a \emph{manifold}, enabling geometric analyses of local stability~\citep{rubin2020manifold}. However, directly constructing and maintaining such a manifold on raw embeddings can be both computationally and memory intensive. Recent work indicates that \emph{graph-based} approaches can capture low-dimensional manifolds within high-dimensional data~\citep{rubin2020manifold}, \emph{especially} when the graph is constructed to preserve meaningful distances. PGMs (or Markov Random Fields) naturally encode these relationships in an undirected graphical structure, allowing for efficient inference about node neighborhoods and global structure \citep{koller2009probabilistic}. Specifically, \citet{feng2021sgl} show that the graph structure learned by PGMs can approximate \emph{resistance distances}, which in turn correlate with Euclidean distances among data samples. Hence, a properly built PGM manifold can reflect \emph{both local and global} geometry—critical for analyzing small perturbations (local stability) and broader connectivity (global structure).  

Despite their theoretical appeal, existing PGM-based methods often rely on iterative optimization or dense computations (e.g., spectral factorization) that become prohibitive for large-scale graphs~\citep{feng2021sgl}. When handling modern NLP datasets—where each sample might represent a document or prompt, and node counts can soar into the hundreds of thousands—these bottlenecks make traditional PGM approaches infeasible.
To address these scalability concerns, we propose a \emph{near-linear complexity} method for building the PGM manifold. Intuitively, we seek a graph Laplacian structure (or precision matrix) that captures the intrinsic geometry of the reduced embeddings (Sec.~\ref{subsec:dim-reduction}) without incurring expensive global factorization steps. 
Recent work \citep{dong2019learning} shows that maximizing a penalized log-likelihood in the form of Eq.~\eqref{opt2} yields a graph topology consistent with the underlying data distribution while preserving essential distance or similarity properties.

% \vspace{1em}
\noindent
\textbf{PGM Objective.}
Let $X \in \mathbb{R}^{|V|\times T}$ be the embedding matrix %\than{what is $k$?}
derived from Sec.~\ref{subsec:dim-reduction}, where each row corresponds to a sample. %\wc{done}\than{maybe add motivation here to smoothly transit to Eqn 1. The general audience maybe unclear about Eqn 1 and why we use it}
We aim to learn a precision matrix $\Theta$ that maximizes~\cite{dong2019learning}:
\begin{equation}\label{opt2}
\max_{\Theta} \quad F(\Theta) 
= \log\det(\Theta) \;-\; \frac{1}{k} \,\operatorname{Tr}\bigl(X^\top \Theta X\bigr),
\end{equation}
subject to $\Theta = \mathcal{L} + \tfrac{1}{\sigma^2}I$, where $\mathcal{L}$ is a valid Laplacian matrix and $\sigma^2$ is a prior variance term. Theorem~\ref{pruning strategy} shows that \emph{maximizing} $F(\Theta)$ can be achieved %in nearly-linear time 
using a spectral sparsification approach, which prunes edges with small distance ratios:
\[
\rho_{p,q} 
= \frac{d^{\mathrm{eff}}(p, q)}{d^{\mathrm{dat}}(p, q)}
= w_{p,q}\,\bigl(d^{\mathrm{eff}}(p, q)\bigr),
\]
where $d^{\mathrm{eff}}(p, q)$ is the effective resistance distance (detailed in Appendix~\ref{sec:appendix_eff_resistance}) between nodes $p$ and $q$%\than{what is effective resistance distance?}
, $d^{\mathrm{dat}}(p, q) = \| X_p - X_q \|_2^2$ is the data distance, and $w_{p,q} = 1/d^{\mathrm{dat}}(p, q)$.

\begin{theorem}
\label{pruning strategy}
Maximizing the objective in Equation~\eqref{opt2} can be done %in nearly-linear time 
via an edge pruning strategy equivalent to spectral sparsification of the initial (dense) graph. Edges with small $\rho_{p,q}$ are removed, preserving the essential spectral (and thus resistance) properties of the original graph. The proof is available in Appendix~\ref{subsec:pruning_sparsification}
\end{theorem}

% \vspace{1em}
\noindent
% \textbf{Initial Graph Construction via \textit{k}-NN.} 
% To apply Theorem~\ref{pruning strategy}, we first build an initial (potentially dense) graph of $|V|$ nodes. For large NLP datasets, a practical approach is to connect each node to its \textit{k} nearest neighbors (in the embedding space from Sec.~\ref{subsec:dim-reduction}), yielding an approximate or exact $k$-NN graph. %This step, while not strictly necessary, often provides a more efficient starting point than a fully dense graph (e.g., every pair of nodes connected).

% \vspace{1em}
\noindent
\textbf{Scalable Spectral Sparsification via Short-Cycle Decomposition.}
A naive implementation of the above pruning requires frequent effective-resistance computations~\cite{spielman2008graph}, which is costly for weighted graphs. Methods such as short-cycle decomposition~\cite{chu2020graph} are effective for \emph{unweighted} graphs but fail to retain accurate resistance distances when weights are discarded. We therefore introduce a refined \emph{spectral sparsification} routine that uses low-resistance-diameter (LRD) decomposition %\than{maybe use a few sentences to briefly describe the LRD? The audience maybe less familiar with such decomposition}\wc{done}
to handle weighted edges without sacrificing the crucial resistance distance information. %In essence, an LRD decomposition partitions a weighted graph into subgraphs of bounded “electrical diameter,” ensuring that each subgraph exhibits small effective resistances internally~\cite{aghdaei2024ingrass}. 
% In an LRD decomposition, the graph is partitioned into multiple small clusters, each with a \emph{bounded effective-resistance diameter}.
\begin{theorem}
\label{LRD}
Our LRD decomposition can efficiently compute the effective resistance of each edge and is effective for sparsifying weighted graphs. The proof is available in Appendix~\ref{sec:proof_LRD_efficiency}
\end{theorem}
%Empirically, we observe that this LRD-based approach more faithfully preserves the manifold structure compared to simple unweighted techniques.

% \vspace{1em}
\noindent
\textbf{From PGM to Manifold.}
With the pruned graph (and correspondingly updated Laplacian $\mathcal{L}$), solving Equation~\eqref{opt2} yields a precision matrix $\Theta$ that encodes the desired topological relationships in $X$. This PGM thus underpins our low-dimensional manifold, accurately maintaining resistance distances for subsequent stability analyses (detailed in Appendix~\ref{sec:appendix_pgm_manifold}). In practice, we initialize the graph with a $k$-nearest-neighbor construction and then apply our near-linear spectral sparsification (as detailed in Section~\ref{subsec:algo-complexity}) %\wc{done}\than{why its near-linear? Did we eblaborate such properties in previous paragraph?}
to achieve scalability. The resulting \emph{manifold} reflects both local and global structures, enabling the DMD calculation.

\subsection{Distance Mapping Distortion (DMD) Calculation}
\label{subsec:dmd-calculation}

%\wc{todo:additional experiment undergoing}\than{we might want to highlight our main innovation here. For example, does it come from DMD? or effective-resistance distance or per-node DMD? We might emphasize our new things. }\wc{sure, todo: highlight our main innovation}
Having constructed the input and output manifolds (Section~\ref{subsec:manifold_construction}), we now introduce the DMD metric~\cite{cheng2021spade} to quantify a model’s robustness at the \emph{sample level}. %Our analysis relies on both $\lambda_{\max}(L_Y^+L_X)$ and $\lambda_{\max}(L_X^+L_Y)$ to capture how local neighborhoods in the input manifold ($G_X$) map to the output manifold ($G_Y$), and vice versa.

\begin{definition}[Distance Mapping Distortion (DMD)]
\label{def:dmd}
Let $F$ be a function mapping an input manifold $G_X=(V,E_X)$ to an output manifold $G_Y=(V,E_Y)$, with $d_X(p,q)$ and $d_Y(p,q)$ denoting the distances between nodes $p$ and $q$ in $G_X$ and $G_Y$, respectively. The distance mapping distortion for $(p,q)$ through $F$ is
\begin{equation}\label{formula_dmd_app}
\gamma^F(p,q) \;\;=\;\; \frac{d_Y(p,q)}{d_X(p,q)}.
\end{equation}
\end{definition}

% DMD captures how “far apart” two nodes in $G_X$ become in $G_Y$. When $d_X(p,q)$ is small (i.e., $p$ and $q$ are close in the input), a large ratio $\gamma^F(p,q)$ indicates significant distortion in the output manifold—often a hallmark of \emph{reduced robustness}. %Conversely, smaller ratios suggest stability under small perturbations in the input space.
%If there exists a pair $(p,q)$ with $d_X(p,q)\!\to\!0$ but $d_Y(p,q)\!\gg\!0$, then arbitrarily small input changes can induce disproportionately large output changes. 
\paragraph{Innovation Highlight:} We show that not only is $\gamma^F_{max} = \max_{p,q} \gamma^F(p,q)$ informative for worst-case local expansion, but also $(\gamma^F_{\min})^{-1} = \bigl(\min_{p,q}\gamma^F(p,q)\bigr)^{-1}$ captures how the model might “collapse” distant inputs into overly similar outputs. We prove in Theorem~\ref{thm:gammaFmin_bound} (below) that large $(\gamma^F_{\min})^{-1}$ implies \emph{another dimension} of poor robustness—distinct from $\gamma^F_{\max}$ (Empirical results are available in Appendix~\ref{sec:appendix_bilip_empirical}).  Hence, both extremes of the distortion spectrum are necessary for a full local analysis.

% \paragraph{Motivating Example.}
% If there exists a pair $(p,q)$ with $d_X(p,q)\to 0$ yet $d_Y(p,q)\gg 0$, then arbitrarily small input changes induce a large output shift ($\gamma^F(p,q)\to\infty$). Conversely, if $d_Y(p,q)\ll d_X(p,q)$ for some $(p,q)$ that are far in $G_X$, the model may “collapse” distinct inputs into almost the same output region, introducing another form of fragility.

\textbf{Effective-resistance distance.}  
To make $\gamma^F$ computationally tractable, we replace geodesic distances with \emph{effective-resistance} ($d^\textit{eff}$). %For a Laplacian $L_G$ of a connected graph $G$, the resistance distance between nodes $p$ and $q$ is given by
%\begin{equation}\label{eq:eff_resist}
%d^\textit{eff}(p,q) \;=\; e_{p,q}^\top L_G^+\, e_{p,q},
%\end{equation}
%where $L_G^+$ is the Moore–Penrose pseudoinverse of $L_G$ and $e_{p,q} = e_p - e_q$ (standard basis vectors).
% \begin{lemma}[Geodesic vs.\ Resistance Distance \cite{chandra1996electrical}]
% \label{lemma:geo_vs_res}
% For any two nodes $p$ and $q$ in a tree, $d^\textit{eff}(p,q)$ matches the geodesic distance $d^\textit{geo}(p,q)$. In more complex (cyclic) graphs, $d^\textit{eff}(p,q)$ is always upper-bounded by $d^\textit{geo}(p,q)$.
% \end{lemma}
$d^\textit{eff}(p,q)$ is always matched or upper-bounded by $d^\textit{geo}(p,q)$~\cite{chandra1996electrical}. 
%\than{for some lemma, if we only list them for complete, but never use it in the later sections. its also possible to directly state the main conclusion in the text (followed by a citation). We listed lots of theorem and lemma in the paper, but some of them might only have loose connection to our main story? We list the theorem/lemma only if they are highly relevant to our approach and form a complete story. Too many theorems might also confuse the audience, and blur our main focus}
Thus, $d^\textit{eff}$ is an \emph{efficient} surrogate for $d^\textit{geo}$, especially when leveraging fast Laplacian solvers \cite{koutis2010approaching,kyng2016approximate}. We then define
\begin{equation}\label{eq:dmd_eff}
\gamma^F \;= \frac{d^\textit{eff}_Y(p,q)}{d^\textit{eff}_X(p,q)} 
\;=\;\frac{\,e_{p,q}^\top L_Y^+\, e_{p,q}\,}{\,e_{p,q}^\top L_X^+\, e_{p,q}\,},
\end{equation}
where $L_X$ and $L_Y$ denote the Laplacians of $G_X$ and $G_Y$, respectively.
Computing $\gamma^F_{max}$ or $\gamma^F_{min}$ \emph{exactly} via Equation~\eqref{eq:dmd_eff} can still be expensive for large graphs, since it involves considering all node pairs $(p,q)$. To alleviate this, \citet{cheng2021spade} proposed a \emph{spectral} upper bound on $\gamma^F_{max}$, termed the $\lambda_{\max}\bigl(L_Y^+\,L_X\bigr)$. Hence, a larger $\lambda_{\max}(L_Y^+L_X)$ suggests a larger distortion ratio and thus poorer robustness. This is also the upper bound of the best Lipschitz constant under the manifold setting~\cite{cheng2021spade}. For $\gamma^F_{min}$ lower bound calculation, we have:

\begin{theorem}\label{thm:gammaFmin_bound}
The minimum distance mapping distortion \(\gamma^F_{min}\) satisfies
\[
\gamma^F_{min} 
\;\;\ge\;\; \frac{1}{\,\lambda_{\max}\bigl(L_X^+\,L_Y\bigr)}.
\]
The proof is available in Appendix~\ref{proof_t1_min}
\end{theorem}

% While $\lambda_{\max}(L_Y^+L_X)$ captures how small \emph{input} changes can balloon in the \emph{output}, the term $\lambda_{\max}(L_X^+L_Y)$ captures the inverse mapping: it identifies whether two \emph{close} points in the output manifold could have been \emph{distant} in the input. In \emph{sample-level} analysis, both can be informative: if $\lambda_{\max}(L_X^+L_Y)$ is also large, it may indicate a different type of manifold distortion (i.e., merging of distinct input regions in the output space).

\paragraph{SALMAN Score.}
We can extend the concept of DMD from node pairs to individual nodes. For each node (sample) $p$, we define $\textbf{SALMAN}^F(p)$:
\begin{equation}
%\textbf{SALMAN}^F(p)=
\frac{1}%{|\mathcal{N}_X(p)|}\sum_{q\in\mathcal{N}_X(p)}
{|\mathcal{N}(p)|}\sum_{q\in\mathcal{N}(p)}
(\gamma^F(p,q)^3+\gamma^F(p,q)^{-3}),
\end{equation}
where $\mathcal{N}_X(p)$ is the set of neighbors of $p$ in $G_X$ and $G_Y$.  A node with a larger $\textbf{SALMAN}^F(p)$ is considered more \emph{fragile}, since its local pairs $(p,q)$ exhibit greater distortion in “expansion” ($\gamma^F(p,q)$) or “collapse” ($1/\gamma^F(p,q)$) senses. To connect expansions and collapses more explicitly, let $\{\lambda_i\}_{i=1}^r$ be the $r$ largest eigenvalues of $L_Y^+L_X$ with corresponding eigenvectors $\{v_i\}_{i=1}^r$, 
and let $\{\mu_i\}_{i=1}^r$ be the $r$ largest eigenvalues of $L_X^+L_Y$ with corresponding eigenvectors $\{w_i\}_{i=1}^r$.
We define the weighted eigensubspace matrices: $V_r=\bigl[v_1\sqrt{\lambda_1},~\dots,~v_r\sqrt{\lambda_r}\bigr]$, $W_r=\bigl[w_1\sqrt{\mu_1},~\dots,~w_r\sqrt{\mu_r}\bigr]$. For each pair $(p,q)$, one has:

\begin{theorem}\label{thm:DMD_relationship}
%We can utilize \texorpdfstring{$L_Y^+L_X$}{Lg} with \texorpdfstring{$L_X^+L_Y$}{Lg} to calculate $\textbf{SALMAN}^F(p)$
% \[
% \lambda_{\max}(L_Y^+L_X) + \lambda_{\max}(L_X^+L_Y)
% \;\;\ge\;\; (\gamma^F_{max}+\frac{1}{\gamma^F_{min}}),
% \]
$\bigl\|\,W_r^\top e_{p,q}\bigr\|_2^2 + \bigl\|\,V_r^\top e_{p,q}\bigr\|_2^2 \propto \gamma^F(p,q)^3+\gamma^F(p,q)^{-3}$. The proof is available in Appendix~\ref{sec:proof_inverse_cubic}
\end{theorem}

\noindent
\textbf{Sample Selection and Correction.}  
Because SALMAN score is computed at the node or node-pair level, we can readily identify “high-distortion” samples and correct them via data augmentation or specialized re-training. This local approach complements global metrics, yielding a holistic robustness analysis pipeline.

\subsection{Complexity}
\label{subsec:algo-complexity}

Our framework has \emph{near-linear} time complexity with respect to the graph size. Below, we briefly outline the main steps and their costs. We first construct a $k$-NN graph from the data points (or embeddings) in $\mathbb{R}^d$. Using modern approximate nearest-neighbor algorithms~\citep{malkov2018efficient} with $O(|V|\,\log|V|)$. \( |V| \) denotes the number of nodes in the graph. Then, We apply a Low-Resistance-Diameter (LRD) approach to sparsify the graph~\citep{koutis2010approaching, cucuringu2016simple}. Let $d$ be the average degree ( small in real-world graphs~\citep{miao2019graph}) and $m$ be the dimension of a Krylov subspace. Then this step runs in $O\bigl(|V|\;d\;m\bigr)$, often simplified to $O\bigl(|V|\;m\bigr)$ under the sparse regime. Evaluating the SALMAN scores for all edges or nodes can be done in $O(|E|)\quad\text{time}$. \( |E| \) denotes the number of edges in the graph. For sparse graphs with $|E|\approx d\,|V|$, this remains near-linear in $|V|$. Experimental results can be found in Appendix~\ref{subsec:scalability_efficiency}.

%Summing these terms yields a near-linear complexity in $|V|$. Hence, our framework efficiently handles large real-world datasets by ensuring that each major step---graph construction, eigensolver computations, spectral sparsification, and SALMAN evaluation---remains \emph{near-linear} relative to $|V|$. This scalability is crucial for practical manifold-based analyses on large dataset.

\section{Experiment Results}
\label{sec:experiments}

We organize our experimental evaluation into three stages, each demonstrating how the \emph{robustness ranking} (derived from the proposed SALMAN measure) can guide practical NLP tasks. The language models used for experiments range from BERT (136M), GPT-2 (1.5B) and the latest Llama3-8B. Details regarding data, hyperparameters, and model architectures are deferred to Appendix~\ref{sec:appendix_exp_setup}

As this is the first work to propose a per-sample NLP robustness ranking, we lack direct comparisons with methods pursuing identical objectives. However, to address the lack of baseline concern, we compare SALMAN against: (1) Euclidean-distance-based ranking, which measures each sample's magnitude of embeddings without local manifold distortion; and (2) Jacobian-based sensitivity analysis. These baselines are simpler proxies for identifying “fragile” points. In Table~\ref{tab:euclidian_jacobian}, we show that both struggle to distinguish robust vs.\ non-robust samples under the same spaCy perturbation scenario.

Additionally, we analyze representative robust versus non-robust samples to confirm that SALMAN reliably identifies vulnerable cases. Moreover, we compare our SALMAN measure against simpler baselines—such as random ranking, a state-of-the-art attack run without our approach, and a state-of-the-art robust training procedure without SALMAN—to assess whether our method provides tangible gains. Our experiments show that SALMAN surpasses these heuristic baselines on diverse perturbation benchmarks, offering strong evidence that SALMAN captures unique facets of sample-level vulnerability. 

% \begin{table}[!h]
% \centering
% \small
% \caption{ Cosine similarity between original vs.\ spaCy-perturbed samples, measured on GPT-2 with Euclidean, Jacobian, and SALMAN rankings. We aim for \textbf{robust} sets to have \textit{higher} similarity and \textbf{non-robust} sets to have \textit{lower} similarity. SALMAN yields the largest gap.}
% \label{tab:euclidian_jacobian}
% \begin{tabular}{l c c c}
% \toprule
% \textbf{Method} & \multicolumn{2}{c}{\textbf{SST-2 / MNLI}} & \textbf{Gap}\\
% \midrule
% Euclidean    & R: 0.9953 / 0.9918~~NR: 0.9986 / 0.9898 & $\rightarrow$ & 0.0033 / 0.0020 \\
% Jacobian     & R: 0.9964 / 0.9942~~NR: 0.9965 / 0.9793 & $\rightarrow$ & 0.0001 / 0.0149 \\
% SALMAN       & R: 0.9981 / 0.9995~~NR: 0.9772 / 0.9730 & $\rightarrow$ & \textbf{0.0209} / \textbf{0.0265} \\
% \bottomrule
% \end{tabular}
% \end{table}

\subsection{Sample Robustness Evaluation}
\label{subsec:robust-eval}

To validate that our robustness ranking meaningfully distinguishes between stable and fragile samples, we subject both \textit{robust} ($1\%$ samples with lowest SALMAN) and \textit{non-robust} ($1\%$ samples with highest SALMAN) samples to various NLP perturbations. 
These perturbations simulate natural edits or noise while controlling for the extent of modification via Levenshtein distance~\cite{ding2021levenshtein}. We thereby ensure that robust and non-robust subsets are perturbed equally in terms of edit cost, allowing a fair comparison of downstream output changes.

Following standard practices in text perturbations~\cite{guo2021towards, ni2024fraud}, we implement three simple but widely used edits: deletion, insertion, and swap. %(1) Deletion: Randomly remove one or more tokens from the input. (2) Insertion: Insert new tokens (e.g., synonyms or filler words) at random positions. (3) Swap: Swap adjacent tokens or partial phrases. 
%Each of these operations introduces small but meaningful disruptions to the text, simulating real-world input noise (typographical errors, token reordering, etc.). 
Following previous works~\cite{le2022perturbations, gupta2023don, jia2023contrastive}, we measure the resultant output embedding shift via cosine similarity between the original and perturbed sentence embeddings, as seen in Table~\ref{tab:cosine_similarity}. 

% \begin{table}[t]
% \centering
% \small
% \caption{Cosine similarities of robust and non-robust samples for different models on the SST-2 and MNLI datasets.}
% \label{tab:cosine_similarity}
% \begin{tabular}{l l c}
% \toprule
% Dataset & Model & Robust/Non-robust\\
%  &  & Cosine Similarity\\
% \midrule
% \multirow{4}{*}{SST-2} 
%  & BERT-base-uncased       & 0.9911/0.8711 \\
%  & RoBERTa-base            & 0.9992/0.9895 \\
%  & DistilBERT-base-uncased & 0.9955/0.9404 \\
%  & ALBERT-base-v2          & 0.9959/0.8279 \\
%  & GPT-2                   & 0.9990 / 0.9153 \\
%  & LLaMA-7B-v2     & 0.9867 / 0.9160 \\
% \midrule
% \multirow{4}{*}{MNLI}
%  & BERT-base-uncased       & 0.9902/0.9410 \\
%  & RoBERTa-base            & 0.9993/0.9926 \\
%  & DistilBERT-base-uncased & 0.9971/0.9650 \\
%  & ALBERT-base-v2          & 0.9953/0.8709 \\
%  & GPT-2                   & 0.9993 / 0.9904 \\
%  & LLaMA-7B-v2     & 0.9925 / 0.9842 \\
% \bottomrule
% \end{tabular}
% \end{table}

Beyond the three basic edits, we employ two state-of-the-art perturbation frameworks (spaCy~\cite{honnibal2020spacy} and TextAttack~\cite{morris2020textattack}) to generate more sophisticated attacks. These methods leverage advanced synonym replacement and gradient-informed edits to produce challenging textual perturbations. Due to the substantial computational overhead of these approaches, we restrict them to two widely recognized LLMs---GPT-2 and LLaMA-7B-v2---thereby striking a balance between experimental rigor and resource feasibility. In Table~\ref{tab:attack_results}, we apply each SOTA method to both robust and non-robust subsets, measuring the resulting cosine similarities to assess susceptibility to adversarial manipulations.

% \begin{table}[t]
% \centering
% \small
% \caption{Robust and non-robust cosine similarities under two different attack methods (spaCy and TextAttack). 
% }
% \label{tab:attack_results}
% \begin{tabular}{lllc}
% \toprule
% Attack & Model & Dataset & Robust/Non-robust  \\
%  &  &  & Cosine Sim.  \\
% \midrule
% \multirow{4}{*}{spaCy}
%  & GPT-2               & SST-2 & 0.9981/0.9772 \\
%  & GPT-2               & MNLI  & 0.9995/0.9730 \\
%  & LLaMA-7B-v2 & SST-2 & 0.9990/0.9751 \\
%  & LLaMA-7B-v2 & MNLI  & 0.9825/0.9612 \\
% \midrule
% \multirow{4}{*}{TextAttack}
%  & GPT-2               & SST-2 & 0.9928/0.9413 \\
%  & GPT-2               & MNLI  & 0.9945/0.9831 \\
%  & LLaMA-7B-v2 & SST-2 & 0.9548/0.8941 \\
%  & LLaMA-7B-v2 & MNLI  & 0.9663 /0.9479     \\
% \bottomrule
% \end{tabular}
% \vspace{-15pt}
% \end{table}

\begin{table*}[t]
\centering
\begin{minipage}[t]{0.45\textwidth}
    \centering
    \small
    \caption{Cosine similarities of robust and non-robust samples for different models on SST-2 and MNLI.}
    \label{tab:cosine_similarity}
    \resizebox{\columnwidth}{!}{%
    \begin{tabular}{l l c}
    \toprule
    Dataset & Model & Robust/Non-robust\\
     &  & Cosine Similarity\\
    \midrule
    \multirow{6}{*}{SST-2} 
     & BERT-base-uncased       & 0.9911/0.8711 \\
     & RoBERTa-base            & 0.9992/0.9895 \\
     & DistilBERT-base-uncased & 0.9955/0.9404 \\
     & ALBERT-base-v2          & 0.9959/0.8279 \\
     & GPT-2                   & 0.9990 / 0.9153 \\
     & LLaMA-7B-v2             & 0.9867 / 0.9160 \\
    \midrule
    \multirow{6}{*}{MNLI}
     & BERT-base-uncased       & 0.9902/0.9410 \\
     & RoBERTa-base            & 0.9993/0.9926 \\
     & DistilBERT-base-uncased & 0.9971/0.9650 \\
     & ALBERT-base-v2          & 0.9953/0.8709 \\
     & GPT-2                   & 0.9993 / 0.9904 \\
     & LLaMA-7B-v2             & 0.9925 / 0.9842 \\
    \bottomrule
    \end{tabular}
    }
\end{minipage}
\hfill
\begin{minipage}[t]{0.5\textwidth}
    \centering
    \small
    \caption{Robust and non-robust cosine similarities under two different attack methods (spaCy \& TextAttack).}
    \label{tab:attack_results}
    \resizebox{\columnwidth}{!}{%
    \begin{tabular}{lllc}
    \toprule
    Attack & Model & Dataset & Robust/Non-robust  \\
     &  &  & Cosine Sim.  \\
    \midrule
    \multirow{4}{*}{spaCy}
     & GPT-2       & SST-2 & 0.9981/0.9772 \\
     & GPT-2       & MNLI  & 0.9995/0.9730 \\
     & LLaMA-7B-v2 & SST-2 & 0.9990/0.9751 \\
     & LLaMA-7B-v2 & MNLI  & 0.9825/0.9612 \\
    \midrule
    \multirow{4}{*}{TextAttack}
     & GPT-2       & SST-2 & 0.9928/0.9413 \\
     & GPT-2       & MNLI  & 0.9945/0.9831 \\
     & LLaMA-7B-v2 & SST-2 & 0.9548/0.8941 \\
     & LLaMA-7B-v2 & MNLI  & 0.9663/0.9479 \\
    \bottomrule
    \end{tabular}
    }
\end{minipage}
\end{table*}

To further assess the difference between robust and non-robust samples, we incorporate two additional metrics: KL-Divergence (KLD) and BERTScore~\cite{zhang2019bertscore}. Table~\ref{tab:bertscore_kld} shows that non-robust samples exhibit larger distribution shift (higher KLD) and lower textual similarity (BERTScore) under perturbations, whereas robust samples remain highly similar.

% \begin{table}[!h]
% \centering
% \small
% \caption{BERTScore and KLD evaluations of robust vs.\ non-robust subsets (GPT-2, SST-2). Higher BERTScore indicates higher textual similarity. KLD measures distribution shift (lower is more stable).}
% \label{tab:bertscore_kld}
% \begin{tabular}{lcccc}
% \toprule
%  & KLD & \multicolumn{3}{c}{BERTScore} \\
% \cmidrule(lr){3-5}
%  &  & Precision & Recall & F1 \\
% \midrule
% Non-Robust & 0.1923 & 0.9961 & 0.9970 & 0.9965 \\
% Robust     & 7.6175e-07 & 0.9992 & 0.9991 & 0.9992 \\
% \bottomrule
% \end{tabular}
% \end{table}

\begin{table*}[t]
\centering
\begin{minipage}[t]{0.46\textwidth}
    \centering
    \small
    \caption{Cosine similarity between original vs.\ spaCy-perturbed samples on GPT-2, for Euclidean-, Jacobian-, and SALMAN-based rankings. We aim for robust sets to have higher similarity and non-robust sets to have lower similarity. SALMAN yields the largest gap.}
    \label{tab:euclidian_jacobian}
    \resizebox{\columnwidth}{!}{%
    \begin{tabular}{l c c c}
    \toprule
    \textbf{Method} & \multicolumn{2}{c}{\textbf{SST-2 / MNLI}} & \textbf{Gap}\\
    \midrule
    Euclidean    
    & R: 0.9953 / 0.9918~~NR: 0.9986 / 0.9898 & $\rightarrow$ & 0.0033 / 0.0020 \\
    Jacobian     
    & R: 0.9964 / 0.9942~~NR: 0.9965 / 0.9793 & $\rightarrow$ & 0.0001 / 0.0149 \\
    SALMAN       
    & R: 0.9981 / 0.9995~~NR: 0.9772 / 0.9730 & $\rightarrow$ & \textbf{0.0209} / \textbf{0.0265} \\
    \bottomrule
    \end{tabular}
    }
\end{minipage}
\hfill
\begin{minipage}[t]{0.46\textwidth}
    \centering
    \small
    \caption{BERTScore and KLD evaluations of robust vs.\ non-robust subsets (GPT-2, SST-2). Higher BERTScore indicates higher textual similarity. KLD measures distribution shift (lower is more stable).}
    \label{tab:bertscore_kld}
    \resizebox{\columnwidth}{!}{%
    \begin{tabular}{lcccc}
    \toprule
     & KLD & \multicolumn{3}{c}{BERTScore} \\
    \cmidrule(lr){3-5}
     &  & Precision & Recall & F1 \\
    \midrule
    Non-Robust & 0.1923 & 0.9961 & 0.9970 & 0.9965 \\
    Robust     & 7.6175e-07 & 0.9992 & 0.9991 & 0.9992 \\
    \bottomrule
    \end{tabular}
    }
\end{minipage}
\end{table*}

By systematically perturbing both robust and non-robust samples, we confirm that non-robust samples consistently exhibit greater output variability under identical input changes. %Robust samples maintain higher fidelity to the original output embeddings, indicating that small textual changes cause minimal semantic drift in the model’s representation. 
This aligns with prior evidence that local text modifications can disproportionately affect certain data points~\cite{morris2020reevaluating}, and it underscores the value of distance mapping distortion in identifying vulnerabilities at the sample level.

\subsection{SALMAN-Guided Attack}
\label{subsec:guided-attack}

Jailbreak Attacks are an important way to assess the security and robustness of LLM~\cite{yi2024jailbreak, chu2024comprehensive}. By strategically crafting prompts, it is possible to bypass the LLM's inherent safeguards and generate harmful content.  Recently, numerous studies have focused on automatically generating stealthy jailbreak attack prompts. However, these current methods are both labor-intensive and computationally demanding. We propose using the SALMAN score to guide more effective attacks.

\noindent
\textbf{Motivation: Find the Non-Robust Data Samples.} We harness the robustness ranking to focus adversarial efforts on the most susceptible samples. This strategy is akin to reducing query complexity in black-box attacks or prioritizing the most “fragile” points. We structured the experiment as follows: 1) we rank the dataset by descending SALMAN score (i.e., from least robust to most robust). 2) We perform the existing attack method \emph{only} on the top $k\%$ of non-robust samples. 3) Then we randomly sample $k\%$ data and use the same method to attack LLM again as a comparison.

For our experiment, we take GCG~\cite{zou2023universal} and AutoDAN~\cite{liu2023autodan} as the jailbreak attack method and use the AdvBench Harmful Behaviors dataset~\cite{zou2023universal} to evaluate the jailbreak attacks. This dataset has 520 data points (Dataset detail in Appendix~\ref{sec:appendix_exp_setup}). After ranking all the data using SALMAN, we selected the top 1\% of unstable samples to launch attacks on LLMs, supplemented by randomly sampling another 1\% of the samples for the same purpose. Subsequently, we evaluated the effectiveness of SALMAN by comparing changes in the Attack Success Rate (ASR) and the number of attack attempts (Steps).

\begin{figure*}[t]
\centering
\begin{subfigure}[t]{0.48\textwidth}
    \centering
    \includegraphics[width=\textwidth]{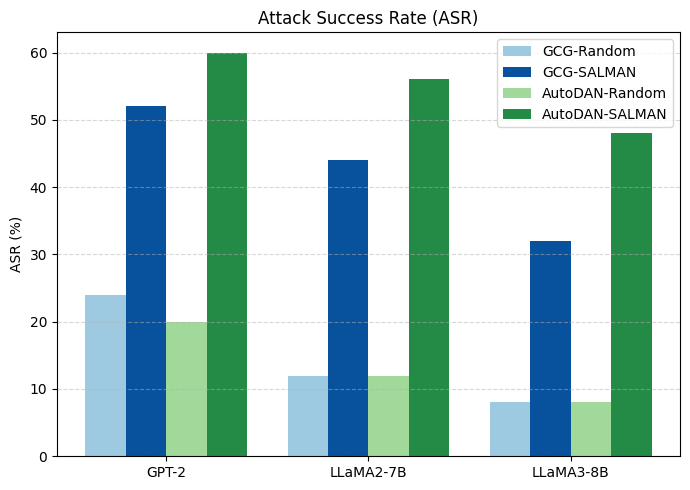}
    \caption{Attack success rate (ASR) and prefix generation steps across different models and attack methods.}
    \label{fig:salman_attack_summary}
\end{subfigure}
\hfill
\begin{subfigure}[t]{0.48\textwidth}
    \centering
    \includegraphics[width=\textwidth]{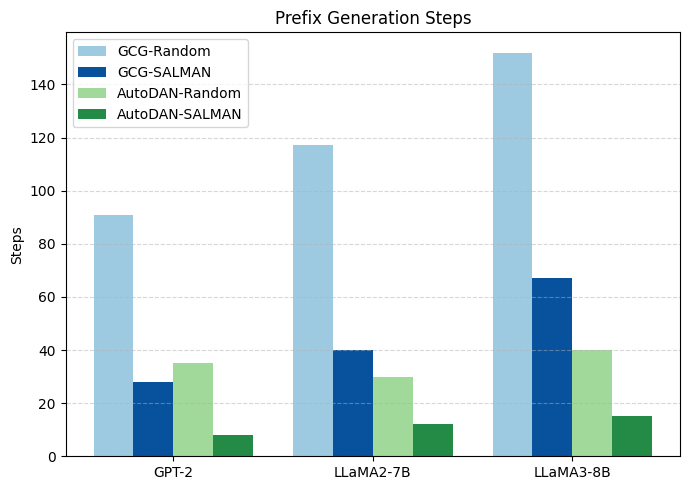}
    \caption{Attacking efficiency: the comparison of the number of attacking attempts.}
    \label{fig:salman_topk_attack}
\end{subfigure}
\caption{Comparison of adversarial attack performance (a) and efficiency (b) with and without SALMAN-guided selection.}
\label{fig:salman_combined}
\end{figure*}

% \subsubsection{Attacks on Top-$k$ Percentiles and Binning Analysis}

Figure~\ref{fig:salman_combined} (a) shows that attacking these low-robustness samples first yields higher success rates \emph{and} reduced time-to-attack compared to random sampling baselines. We visualized attacking efficiency in Figure~\ref{fig:salman_combined} (b). The SALMAN-based ranking serves as an efficient “shortcut” for adversarial testing. Then, we investigate SALMAN ranking by selecting different top-$k$ percentages of the dataset. We bin (10\%) the entire dataset into deciles by SALMAN rank and apply the same attack methods to each bin. As $k$ increases, we include more (relatively) robust samples, resulting in lower overall ASR and efficiency. We further assessed the SALMAN-Guided attack using proxy models to evaluate the robustness of the proposed method. The experimental results are presented in Table~\ref{tab:proxy_salman}.

% \begin{table}[!h]
% \centering
% \small
% \caption{Attack success rates when applying GCG and AutoDan to the top 5\% or 10\% of non-robust samples in the AdvBench dataset.}
% \label{tab:percentile_attack}
% \resizebox{0.5\columnwidth}{!}{
% \begin{tabular}{lccc}
% \toprule
% \textbf{Method-Top$k$} & \textbf{GPT-2} & \textbf{LLaMA2-7B} & \textbf{LLaMA3-8B} \\
% \midrule
% GCG-5\%   & 56\% & 44\% & 28\% \\
% GCG-10\%  & 30.76\% & 26.97\% & 21.15\% \\
% AutoDAN-5\%   & 46.13\% & 42.30\% & 40\% \\
% AutoDAN-10\%  & 44\% & 34.61\% & 28.85\% \\
% \bottomrule
% \end{tabular}
% }
% \end{table}

\subsection{SALMAN-Guided LLM Fine-Tuning}
\label{subsec:guided-finetuning}

Fine-tuning LLMs sometimes leads to overfitting, losing key representations from pre-training~\cite{howard2018universal, neerudu2023robustness}. Prior work has attempted to measure how much an LLM’s internal representations drift from the pre-trained checkpoint, using similarity metrics such as CKA or STIR~\cite{ neerudu2023robustness}. A large drift often indicates the model is overfitting, thus sacrificing generalization and robustness in practice.

\noindent
\textbf{Motivation: Focus on Non-Robust Data.}
Several studies show that focusing on non-robust samples during training can improve model robustness and generalizability~\citep{cheng2021spade,zhu2023improving}. Inspired by these findings, we propose to \emph{down-weight} robust samples and \emph{up-weight} non-robust samples (as determined by our SALMAN-based ranking) when fine-tuning an LLM. The intuition is that easy/robust data rarely contributes to boosting generalizable features, whereas harder (high DMD) data pushes the model to learn more discriminative patterns.

We follow the fine-tuning protocol described by \citet{neerudu2023robustness}. By placing greater emphasis on non-robust data (as detailed in Appendix~\ref{sec:appendix_weighting})
, we hypothesize that the fine-tuned model retains more generalizable features from its pre-training, avoiding over-specialization to easy examples. We observe two key outcomes:
\begin{enumerate}
    \item \textit{Comparable Performance, Closer to Pre-training:} %Despite redistributing training weights, the LLMs obtained through \emph{SALMAN-guided} fine-tuning  achieve the same or better downstream accuracy as a standard fine-tuned baseline \emph{yet} exhibit higher similarities to the pre-trained checkpoint. Across multiple GLUE tasks, we find up to $54\%$ improvement in CKA or STIR scores, indicating less representational drift. Due to space constraints, we present the additional results in Appendix~\ref{sec:appendix_gpt2_sst2_rte} while highlighting the key findings in Table~\ref{tab:gpt2_cola_small}.
    SALMAN-guided LLMs achieve comparable or better accuracy vs.\ standard fine-tuning, yet exhibit higher similarity to the pre-trained checkpoint. On GLUE tasks, we find up to 54\% gains in CKA or STIR, signifying less drift with better accuracy. Results appear in Appendix~\ref{sec:appendix_gpt2_sst2_rte} and highlight key findings.

\begin{table*}[t]
\centering
\label{tab::stir_robustness_summary}
%--------------------- Left side ----------------------%
\begin{minipage}[t]{0.47\textwidth}
    \vspace{-1pt} % Ensure alignment at the top

    %------------------- Top table (left) -------------------%
    \begin{minipage}[t]{\textwidth}
        \centering
        \small
        \caption{Proxy-based SALMAN ranking vs.\ target LLM under GCG. We list the ASR on the top-1\% non-robust subsets identified by the proxy model.}
        \label{tab:proxy_salman}
        \resizebox{\columnwidth}{!}{
            \begin{tabular}{lccc}
            \toprule
            \textbf{Proxy} & GPT-2 & LLaMA2-7B & LLaMA3-8B \\
            \midrule
            GPT-2      & 60\% & 48\% & 40\% \\
            LLaMA2-7B  & 60\% & 56\% & 40\% \\
            LLaMA3-8B  & 60\% & 56\% & 48\% \\
            \bottomrule
            \end{tabular}
        }
    \end{minipage}

    \vspace{1em} % some vertical spacing between top and bottom tables

    %------------------- Bottom table (left) -------------------%
    \begin{minipage}[t]{\textwidth}
        \centering
        \small
        \caption{Performance of ROSE fine-tuning on SST-2, RTE, and QNLI tasks with RoBERTa\textsubscript{BASE}. Each cell shows \textit{ROSE fine-tuning / SALMAN-guided ROSE fine-tuning} accuracy. Bold is better.}
        \label{tab:rose_guided_transposed_base}
        \begin{tabular}{lcc}
        \toprule
        Task & GLUE & AdvGLUE \\
        \midrule
        SST-2 & 94.84 / \textbf{94.84} & 37.67 / \textbf{41.22} \\
        RTE   & 78.34 / \textbf{78.34} & 35.49 / \textbf{40.75} \\
        QNLI  & 92.19 / \textbf{92.81} & 44.19 / \textbf{48.02} \\
        \bottomrule
        \end{tabular}
    \end{minipage}
\end{minipage}
\hfill
%--------------------- Right side ----------------------%
\begin{minipage}[t]{0.47\textwidth}
    \centering
    \scriptsize
    \caption{Model-level robustness score~\cite{neerudu2023robustness} of BERT and GPT-2 on CoLA
    and SST-2 under various perturbations. 
    Each cell shows the \emph{Normal/SALMAN-guided} fine-tuning value. Better results are in \textbf{bold}.}
    \setlength{\tabcolsep}{2pt}
    \label{tab:cola_sst2_mnli_rte_small}
    \begin{tabular}{lcc}
    \toprule
    \textbf{Perturbation} & \textbf{CoLA} & \textbf{SST-2} \\
    \midrule
    \multicolumn{3}{c}{\textit{BERT}} \\
    Drop nouns   & 0.18 / \textbf{0.29} & 0.92 / \textbf{1.02}  \\
    Drop verbs   & 0.05 / \textbf{0.22} & 0.95 / \textbf{1.03}  \\
    Drop first   & 0.48 / \textbf{0.70} & 0.98 / \textbf{1.01}  \\
    Drop last    & 0.34 / \textbf{0.72} & 1.00 / \textbf{1.00}  \\
    Swap text    & 0.13 / \textbf{0.22} & 0.98 / \textbf{1.01}  \\
    Add text     & 0.85 / \textbf{0.88} & 0.99 / \textbf{0.99}  \\
    Change char  & 0.14 / \textbf{0.24} & 0.84 / \textbf{0.91}  \\
    Bias         & 0.95 / \textbf{0.99} & 1.00 / \textbf{1.00}  \\
    \midrule
    \multicolumn{3}{c}{\textit{GPT-2}} \\
    Drop nouns   & 0.10 / \textbf{0.47} & 0.93 / \textbf{1.00}  \\
    Drop verbs   & 0.24 / \textbf{0.32} & 0.95 / \textbf{1.00}  \\
    Drop first   & 0.75 / \textbf{0.91} & 0.97 / \textbf{1.01}  \\
    Drop last    & 0.45 / \textbf{0.78} & 0.99 / \textbf{1.01}  \\
    Swap text    & 0.16 / \textbf{0.45} & 0.98 / \textbf{1.00}  \\
    Add text     & 0.92 / \textbf{0.96} & 0.99 / \textbf{1.01}  \\
    Change char  & 0.29 / \textbf{0.36} & 0.86 / \textbf{1.02}  \\
    Bias         & 0.96 / \textbf{1.14} & \textbf{1.01} / 0.99  \\
    \bottomrule
    \end{tabular}
\end{minipage}
\vspace{-5pt}
\end{table*}

\item \textit{Enhanced Robustness Scores:} 
Although our paper introduces a \emph{sample-centric} robustness measure for \emph{ranking} individual samples \emph{without} requiring explicit perturbations, we also need to assess the \emph{entire model}’s robustness after fine-tuning.  
To this end, we adopt the model-level robustness score proposed by \citet{neerudu2023robustness}, which measures how representations change under various text perturbations across the full dataset. Measuring each model’s robustness scores confirms that the \emph{SALMAN-guided} LLM obtains higher robustness than a conventionally fine-tuned model. As shown in Table~\ref{tab:cola_sst2_mnli_rte_small}, we attribute this improvement to the heightened focus on challenging (non-robust) samples during training.

% \begin{table}[t]
% \centering
% \small
% \caption{Model level robustness score~\cite{neerudu2023robustness} of BERT and GPT-2 on CoLA (Matthews CC) and SST-2 (Accuracy) under various perturbations. 
% Each cell shows the \textit{Normal/SALMAN-guided} fine-tuning value. Better results are in \textbf{bold}.}
% \label{tab:cola_sst2_mnli_rte_small}
% % \resizebox{\columnwidth}{!}{%
% \begin{tabular}{c c c}
% \toprule
% Perturbation & CoLA & SST-2 \\
% \midrule
% \multicolumn{3}{c}{BERT} \\
% Drop nouns   & 0.18 / \textbf{0.29} & 0.92 / \textbf{1.02}  \\
% Drop verbs   & 0.05 / \textbf{0.22} & 0.95 / \textbf{1.03}  \\
% Drop first   & 0.48 / \textbf{0.70} & 0.98 / \textbf{1.01}  \\
% Drop last    & 0.34 / \textbf{0.72} & 1.00 / \textbf{1.00}   \\
% Swap text    & 0.13 / \textbf{0.22} & 0.98 / \textbf{1.01}  \\
% Add text     & 0.85 / \textbf{0.88} & 0.99 / \textbf{0.99} \\
% Change char  & 0.14 / \textbf{0.24} & 0.84 / \textbf{0.91}  \\
% Bias         & 0.95 / \textbf{0.99} & 1.00 / \textbf{1.00} \\
% \midrule
% \multicolumn{3}{c}{GPT-2} \\
% Drop nouns   & 0.10 / \textbf{0.47} & 0.93 / \textbf{1.00}  \\
% Drop verbs   & 0.24 / \textbf{0.32} & 0.95 / \textbf{1.00}  \\
% Drop first   & 0.75 / \textbf{0.91} & 0.97 / \textbf{1.01} \\
% Drop last    & 0.45 / \textbf{0.78} & 0.99 / \textbf{1.01}  \\
% Swap text    & 0.16 / \textbf{0.45} & 0.98 / \textbf{1.00} \\
% Add text     & 0.92 / \textbf{0.96} & 0.99 / \textbf{1.01}  \\
% Change char  & 0.29 / \textbf{0.36} & 0.86 / \textbf{1.02} \\
% Bias         & 0.96 / \textbf{1.14} & \textbf{1.01} / 0.99  \\
% \bottomrule
% \end{tabular}%
% % }
% \end{table}

\end{enumerate}

\noindent
\textbf{Combining with SOTA Robust Training.}
We further integrate our approach with ROSE~\citep{jiang2022rose}, a selective fine-tuning framework that prunes “spurious” parameter updates to achieve greater adversarial resilience. Specifically, we embed our SALMAN-based weighting (as detailed in Appendix~\ref{sec:appendix_weighting})
within ROSE’s parameter selection process. Experimental results in Table~\ref{tab:rose_guided_transposed_base} reveal that sample-level weighting and parameter-level selection are complementary strategies.

% \begin{table}[t]
% \centering
% \small
% \caption{Performance of ROSE fine-tuning on SST-2, RTE, and QNLI tasks with RoBERTa\textsubscript{BASE}. Each cell shows \textit{ROSE fine-tuning / SALMAN-guided ROSE fine-tuning} accuracy values. Better results are in \textbf{bold}.}
% \label{tab:rose_guided_transposed_base}
% %\resizebox{\columnwidth}{!}{%
% \begin{tabular}{lcccc}
% \toprule
% Task & \multicolumn{2}{c}{RoBERTa\textsubscript{BASE}} \\
% \cmidrule(lr){2-3}
%      & GLUE & AdvGLUE \\
% \midrule
% SST-2 & 94.84 / \textbf{94.84} & 37.67 / \textbf{41.22} \\
% RTE   & 78.34 / \textbf{78.34} & 35.49 / \textbf{40.75} \\
% QNLI  & 92.19 / \textbf{92.81} & 44.19 / \textbf{48.02} \\
% \bottomrule
% \end{tabular}%
% %}
% \end{table}

% \end{itemize}

%By prioritizing non-robust samples during LLM fine-tuning, we mitigate the risk of overfitting to trivial data and preserve more of the pre-trained representation structure. This strategy (1) matches conventional fine-tuning in overall accuracy, (2) yields a model closer to its pre-training in terms of layer similarity, and (3) enhances LLM robustness. Moreover, it integrates seamlessly with advanced selective methods like ROSE, delivering further improvements in resilience against adversarial inputs.

\section{Conclusion}
\label{sec:conclusion}

We introduced SALMAN, a novel measure to identify and rank the local robustness of transformer-based language models. Our experiments across diverse benchmarks and large language models reveal that SALMAN not only distinguishes \emph{robust} from \emph{non-robust} samples under both simple and SOTA perturbations, but also effectively guides attacks and fine-tuning. Moreover, incorporating SALMAN into the existing robust training framework yields even greater resilience against adversarial perturbations. These results underscore the potential of leveraging \emph{sample-level robustness} to bolster both attack strategies and robust model adaptation. 

%\paragraph{Future Work.}
%We plan to investigate the transferability of SALMAN-based rankings across domains, scale up to ultra-large models in distributed settings, and evaluate robustness in more complex tasks (e.g., dialogue, multimodal reasoning). We believe that sample-level robustness analysis opens a promising avenue toward systematically enhancing the reliability of modern transformer-based LLMs.

% \section*{References}

% References follow the acknowledgments in the camera-ready paper. Use unnumbered first-level heading for
% the references. Any choice of citation style is acceptable as long as you are
% consistent. It is permissible to reduce the font size to \verb+small+ (9 point)
% when listing the references.
% Note that the Reference section does not count towards the page limit.
% \medskip

\newpage

{
\small

% \nocite{langley00}

\bibliographystyle{unsrtnat}
\bibliography{example_paper}
}

%%%%%%%%%%%%%%%%%%%%%%%%%%%%%%%%%%%%%%%%%%%%%%%%%%%%%%%%%%%%

\newpage
\appendix

\section{Technical Appendices and Supplementary Material}
\subsection{Deterministic Hidden-State Embeddings}
\label{sec:appendix_deterministic}

Though modern transformers can produce \emph{stochastic} outputs at the token level (e.g., due to beam search, random sampling, or dropout), their internal hidden states can remain \emph{largely deterministic} under fixed conditions~\citep{ wolf2020transformers}. Below, we validate this claim by comparing token-level embeddings and pooled hidden-state embeddings across different decoding strategies. We then observe what happens when we additionally fix the seed.

\paragraph{Token vs.\ Pooled Embeddings Under Varying Seeds.}
We feed the \emph{same} input sequence through various transformers (DistilBERT, BERT, RoBERTa, and Google-Electra), each time \emph{without} enforcing a fixed random seed for decoding. We then collect:

\begin{itemize}
    \item Token Embeddings. The final output embeddings for each token in the decoded sequence (i.e., after language modeling head).
    \item Pooled MHSA Output Embeddings. Our approach aggregates multiple attention heads and pools them into a single output vector per sequence, thereby abstracting away token-level variations.
\end{itemize}

For each model, we compute the cosine similarity between embeddings arising from different decoding runs of the \emph{same} input. Table~\ref{tab:cosine_stability} shows representative results for three datasets: SQuAD, IMDB, and AG-News.

\begin{table}[h]
\centering
\caption{Cosine Similarity of Embeddings Across Different Decoding Runs \emph{Without} a Fixed Random Seed. Higher is more stable.}
\label{tab:cosine_stability}
\vspace{1ex}
\begin{tabular}{lcccc}
\toprule
& \multicolumn{4}{c}{Token Embeddings} \\
\cmidrule(lr){2-5}
Dataset & DistilBERT & BERT & RoBERTa & Google-Electra \\
\midrule
SQuAD  & 0.9338 & 0.2302 & 0.9967 & 0.1859 \\
IMDB   & 0.9685 & 0.4443 & 0.9977 & 0.5131 \\
AG-News& 0.9450 & 0.5967 & 0.9954 & 0.5771 \\
\bottomrule \\
& \multicolumn{4}{c}{Pooled MHSA Output Embeddings} \\
\cmidrule(lr){2-5}
Dataset & DistilBERT & BERT & RoBERTa & Google-Electra \\
\midrule
SQuAD  & 1.00 & 1.00 & 1.00 & 1.00 \\
IMDB   & 1.00 & 1.00 & 1.00 & 1.00 \\
AG-News& 0.99 & 0.99 & 0.99 & 0.99 \\
\bottomrule
\end{tabular}
\end{table}

Even when seeds vary, \textbf{token-level embeddings} exhibit inconsistent cosine similarity across runs (e.g., BERT scoring only $0.23$ for SQuAD). By contrast, \textbf{our pooled MHSA method} maintains consistently high similarity (0.99 or 1.00), indicating a more stable representation that does not fluctuate with token-level decoding choices. This stability suggests that the representation potentially captures a more consistent global semantic space~\cite{reimers2019sentence}.

\paragraph{Fixed Seed + Pooled MHSA}

Finally, we fix the random seed (and disable dropout) for all runs using our pooled MHSA approach, ensuring the only factor causing embedding changes is an explicit \emph{input} perturbation (e.g., synonyms swapped). Under \emph{identical} inputs and the same seed, the pooled MHSA output embeddings always match exactly (cosine similarity = 1.00), regardless of how tokens might be sampled. As summarized in Table~\ref{tab:cosine_stability_seedfix}, \textbf{all entries become $1.00$} when there is \emph{no input perturbation}.

\begin{table}[h!]
\centering
\caption{Cosine Similarity with Fixed Seed \emph{and} Pooled MHSA. Identical inputs yield identical embeddings (similarity = 1.00).}
\label{tab:cosine_stability_seedfix}
\vspace{1ex}
\begin{tabular}{lcccc}
\toprule
& \multicolumn{4}{c}{Seed-Fixed Pooled MHSA Output} \\
\cmidrule(lr){2-5}
Dataset & DistilBERT & BERT & RoBERTa & Google-Electra \\
\midrule
SQuAD   & 1.00 & 1.00 & 1.00 & 1.00 \\
IMDB    & 1.00 & 1.00 & 1.00 & 1.00 \\
AG-News & 1.00 & 1.00 & 1.00 & 1.00 \\
\bottomrule
\end{tabular}
\end{table}

\subsection{Proof for Theorem~\ref{pruning strategy}}
\label{subsec:pruning_sparsification}

We now show that \emph{maximizing} the objective
\begin{equation}\label{eq:opt2_copy}
\max_{\Theta} \quad F(\Theta) \;=\; \log\det(\Theta) \;-\; \frac{1}{k}\,\mathrm{Tr}\bigl(X^\top \Theta \,X\bigr),
\end{equation}
where $\Theta = \mathcal{L} + \tfrac{1}{\sigma^2}I$, can be achieved by removing (or down-weighting) edges whose \emph{distance ratio} is small. In essence, these low-ratio edges contribute less to $\log\!\det(\Theta)$ while incurring a larger penalty in the trace term, so pruning them increases $F(\Theta)$.

\paragraph{1.\; Decomposing the Objective.}
Writing $\mathcal{L} = \sum_{(p,q)\in E} w_{p,q}\,e_{p,q}\,e_{p,q}^\top,$
we split $F(\Theta)$ into two terms:
\[
F(\Theta)\;=\;F_1(\Theta)\;-\;\frac{1}{k}\,F_2(\Theta),
\quad
\text{where}
\]
\[
F_1(\Theta)\;=\;\log\!\det(\Theta),
\quad
F_2(\Theta)\;=\;\mathrm{Tr}\!\bigl(X^\top \Theta\, X\bigr).
\]
Since $\Theta = \mathcal{L} + \tfrac{1}{\sigma^2} I$, each edge weight $w_{p,q}$ appears explicitly in $\mathcal{L}$.

\paragraph{2.\; Gradient with Respect to an Edge Weight.}
To optimize $F(\Theta)$ w.r.t.\ a single edge weight $w_{p,q}$:

\begin{itemize}
\item \textbf{Term $F_1(\Theta)$}:  
  Let $\lambda_i$ be the $i$-th eigenvalue of $\mathcal{L}$, and $v_i$ its eigenvector. Then
  \[
  \frac{\partial}{\partial w_{p,q}}\!\Bigl(\log\!\det(\Theta)\Bigr)
  \;=\;
  \frac{\partial}{\partial w_{p,q}}\!
  \Bigl[\,
    \log\!\det\!\bigl(\mathcal{L} + \tfrac{1}{\sigma^2}I\bigr)
  \Bigr].
  \]
  By standard matrix calculus, this derivative can be linked to the \emph{effective resistance distance} $d^{\mathrm{eff}}(p,q)$:
  \[
  \frac{\partial F_1}{\partial w_{p,q}}
  \;\;\approx\;\;d^{\mathrm{eff}}(p,q),
  \]
  where $d^{\mathrm{eff}}(p,q)$ encapsulates how strongly edge $(p,q)$ influences $\log\!\det(\Theta)$.

\item \textbf{Term $F_2(\Theta)$}:
  \[
  F_2(\Theta)
  =\mathrm{Tr}\!\bigl(X^\top\Theta X\bigr)
  =\mathrm{Tr}\!\bigl(X^\top (\mathcal{L} + \tfrac{1}{\sigma^2}I)\,X\bigr)
  =\frac{\mathrm{Tr}(X^\top X)}{\sigma^2}
   \;+\;\sum_{(p,q)\in E}
   w_{p,q}\,\bigl\|X^\top e_{p,q}\bigr\|_2^2.
  \]
  Since $\bigl\|X^\top e_{p,q}\bigr\|_2^2 = \|X_p - X_q\|_2^2 = d^{\mathrm{dat}}(p,q)$, we have
  \[
  \frac{\partial F_2}{\partial w_{p,q}}
  \;=\;\bigl\|X_p - X_q\bigr\|_2^2
  =\;d^{\mathrm{dat}}(p,q).
  \]
  Furthermore, $d^{\mathrm{dat}}(p,q) = \frac{1}{w_{p,q}}$, which implies
  \[
  \frac{\partial F_2}{\partial w_{p,q}}
  \;=\;\frac{1}{\,w_{p,q}\,}.
  \]
\end{itemize}

Hence, the derivative of $F(\Theta)=F_1-\tfrac{1}{k}F_2$ w.r.t.\ $w_{p,q}$ is
\begin{equation}\label{eq:gradF}
\frac{\partial F}{\partial w_{p,q}}
\;=\;
d^{\mathrm{eff}}(p,q)\;-\;\frac{1}{k}\,\frac{1}{\,w_{p,q}\,}.
\end{equation}

\paragraph{3.\; Distance Ratio and Pruning Condition.}
Rewriting Equation~\eqref{eq:gradF}:
\[
d^{\mathrm{eff}}(p,q)
\;-\;
\frac{1}{k}\,\frac{1}{\,w_{p,q}\,}
=0
\;\;\Longleftrightarrow\;\;
d^{\mathrm{eff}}(p,q)
\;=\;\frac{1}{\,k\,}\,\frac{1}{w_{p,q}}.
\]
Define the \emph{distance ratio} for edge $(p,q)$:
\[
\rho_{p,q}
=\;\frac{\,d^{\mathrm{eff}}(p,q)\,}{\,d^{\mathrm{dat}}(p,q)\,}
=\;w_{p,q}\,\bigl(d^{\mathrm{eff}}(p,q)\bigr).
\]
When $d^{\mathrm{eff}}(p,q)$ is relatively large compared to $\tfrac{1}{w_{p,q}}$, we have $\rho_{p,q}$ large, indicating an important edge for $\log\!\det(\Theta)$. Conversely, if $\rho_{p,q}$ is small, the edge $(p,q)$ contributes little to $F_1(\Theta)$ but increases $F_2(\Theta)$, thereby reducing $F(\Theta)$. 

\paragraph{4.\; Conclusion: Prune Low-Ratio Edges.}
Thus, maximizing \eqref{eq:opt2_copy} naturally leads to \emph{removing or down-weighting} edges whose ratio
\[
\rho_{p,q}
\;=\;
\frac{\,d^{\mathrm{eff}}(p,q)\,}{\,d^{\mathrm{dat}}(p,q)\,}
\]
is below a certain threshold. By pruning these edges, we preserve the essential spectral structure needed to keep $\log\!\det(\Theta)$ high (reflecting higher effective resistance) while mitigating the penalty in $\operatorname{Tr}(X^\top \Theta X)$ from edges that have large data distance but small effective resistance. In other words, \emph{edges with large $\rho_{p,q}$ stay}, and edges with small $\rho_{p,q}$ are pruned, thereby maximizing $F(\Theta)$ and maintaining the key (Laplacian) properties of the original graph.

\subsection{Proof of LRD Decomposition for Efficient Edge Resistance Computation and Weighted Graph Sparsification}
\label{sec:proof_LRD_efficiency}

In this section, we establish that the \emph{low-resistance-diameter} (LRD) decomposition scheme can efficiently approximate the effective resistance for each edge in a \emph{weighted} graph and thus provide an effective path toward spectral sparsification. Our argument proceeds in two stages:

\begin{enumerate}
    \item \textbf{Approximate Effective Resistance via Krylov Subspaces:} We show how the iterative procedure yields reliable estimates of $d^{\mathrm{eff}}(p,q)$ in near-linear time.
    \item \textbf{Bound Cycle Lengths under LRD Clustering:} We explain how the multilevel contraction and supernode formation ensure that edges with large resistance distances are effectively sampled or retained, yielding a final sparsified graph that spectrally approximates the original.
\end{enumerate}

\paragraph{Stage 1: Approximating Effective Resistances via Krylov Subspaces.}

The resistance distance for an edge $(p,q)$ in a graph $G=(V,E)$ with Laplacian $L_G$ can be expressed as:
\[
d^{\mathrm{eff}}(p,q) 
\;=\;
\sum_{i=2}^{N}
\frac{\bigl(u_i^\top e_{p,q}\bigr)^2}{\,u_i^\top L_G\,u_i\,},
\]
where $u_2, \dots, u_N$ are the (nontrivial) eigenvectors of $L_G$ and $e_{p,q}=e_p-e_q$. Directly computing all eigenvalues/eigenvectors for large $G$ is typically prohibitive. Instead, we replace these eigenvectors with a small set of Krylov basis vectors $x^{(1)}, x^{(2)}, \ldots, x^{(m)}$, which approximate the subspace spanned by the top spectral components of $L_G$. Specifically, each $x^{(i)}$ is drawn from
\[
\kappa_{m}(A,\,c) \;=\; \mathrm{span}\{\,c,\;A\,c,\;A^2\,c,\;\ldots,\;A^{m-1}\,c\,\},
\]
where $A$ is the adjacency matrix and $c$ is a random vector. Orthogonalizing and normalizing these $m$ vectors ensures a concise basis in which to project $e_{p,q}$.

\begin{lemma}[Krylov Approximation of Effective Resistance]
\label{lemma:krylov}
Suppose $x^{(1)}, \dots, x^{(m)}$ are $m$ orthonormal vectors approximating the dominant spectral subspace of $L_G$ (via a Krylov process). Then for any edge $(p,q)\in E$,
\[
d^{\mathrm{eff}}(p,q)
\;\approx\;
\sum_{i=1}^{m}
\frac{\bigl(x^{(i)\top} e_{p,q}\bigr)^2}
     {\,x^{(i)\top}\,L_G\,x^{(i)}\,}.
\]
Choosing $m=\widetilde{O}(\log N)$ and updating each level in near-linear time yields high-probability error bounds comparable to exact spectral decompositions~\citep{spielman2011graph,koutis2010approaching}.
\end{lemma}

\paragraph{Stage 2: LRD-based Short-Cycle Decomposition for Weighted Graphs.}

The second key step is the \emph{multilevel} contraction scheme that ensures edges with large effective resistance remain “visible” at higher levels, while short cycles (or low-resistance edges) are contracted to form supernodes. Specifically:

\begin{itemize}
\item At level $\delta$, each edge $(p,q)$ is either \emph{contracted} (if $d_{eff}^{(\delta)}(p,q)$ is below the chosen threshold) or \emph{retained} (if $d_{eff}^{(\delta)}(p,q)$ is above the threshold). Contraction merges $p$ and $q$ into a supernode $\vartheta$, assigning it an accumulated weight $\eta_\vartheta$ via:
\begin{equation}\label{eq:Nweights}
    \eta_\vartheta := \eta(p^{(\delta)}) + \eta(q^{(\delta)}) + d_{eff}^{(\delta)} (p,q).
\end{equation}
\item As edges are contracted, any cycles formed at level $\delta$ have length bounded by the effective-resistance diameter. Consequently, short cycles in the \emph{weighted} setting are handled similarly to \citet{chu2020graph}’s unweighted approach, except that we measure distances via $d_{eff}$, not just hop counts.
\item After finalizing the clusters (supernodes), the “inter-cluster” edges (those bridging different clusters) are preserved or upweighted in the sparsified graph. These edges typically have higher $d_{eff}(p,q)$ and thus significantly impact spectral properties of $L_G$.
\end{itemize}

Formally, let $L_H$ denote the Laplacian of the sparsified graph $H$ returned by the LRD decomposition. We say $H$ is a \emph{$(1\pm\varepsilon)$-spectral-approximation} of $G$ if, for all $x\in \mathbb{R}^N$,
\[
(1-\varepsilon)\,x^\top L_G\,x 
\;\le\;
x^\top L_H\,x
\;\le\;
(1+\varepsilon)\,x^\top L_G\,x.
\]
Standard arguments from spectral sparsification \citep{spielman2011graph} show that any procedure ensuring accurate effective-resistance estimates can preserve the graph’s quadratic form up to $(1\pm\varepsilon)$ factors. The main difference in \emph{our} approach is the use of low-resistance-diameter cycles instead of purely unweighted short cycles.

% \begin{figure}[H]
%     \centering
%     \includegraphics[width=0.7\textwidth, trim={2cm 12.5cm 30cm 17cm}, clip]{figures/sample_case.pdf}
%     \vspace{-10pt}
%     \caption{The proposed spectral sparsification algorithm. (a) The initial graph. (b) LRD decomposition for graph clustering. (c) LSSTs for pruning non-critical edges within clusters. (d) The final graph-based manifold with two inter-cluster edges. }
% \label{figure:SC_SP}
% \end{figure}

\begin{theorem}[LRD for Weighted Graph Sparsification]
\label{thm:LRD_sparsify}
Let $G=(V,E)$ be a connected weighted graph with $N$ nodes and $M$ edges, and let $0<\varepsilon<1$ be a chosen approximation factor. Then, by applying the LRD-based spectral sparsification algorithm %(Figure~\ref{figure:SC_SP}) 
with Krylov-based effective-resistance estimates:
\begin{enumerate}
    \item We can approximate $d^{\mathrm{eff}}(p,q)$ for all edges $(p,q)\in E$ in near-linear time (Lemma~\ref{lemma:krylov}).
    \item We contract short cycles (below a chosen $d_{eff}$ threshold) and preserve inter-cluster edges with high $d_{eff}(p,q)$, forming a sparsified graph $H$ with Laplacian $L_H$.
    \item With high probability, $H$ satisfies $ (1-\varepsilon)\,x^\top L_G x \;\le\; x^\top L_H x \;\le\;(1+\varepsilon)\,x^\top L_G x,\;\forall x\in\mathbb{R}^N.$
\end{enumerate}
Hence, $H$ serves as a $(1\pm\varepsilon)$-spectral-approximation to $G$, yielding a low-complexity graph that closely preserves the original graph’s spectral (and thus effective-resistance) structure.
\end{theorem}

\begin{proof}[Proof Sketch]
\quad

\noindent \emph{(1) Effective-resistance approximation.}  
By Lemma~\ref{lemma:krylov}, each edge’s resistance can be estimated via $m=\widetilde{O}(\log N)$ Krylov vectors per level of the hierarchy. Summed over $\delta$ levels, the total cost remains near-linear in $N+M$ (plus polylogarithmic factors), similar to \citet{koutis2010approaching,kyng2016approximate}.

\noindent \emph{(2) Cycle decomposition.}  
Following \cite{chu2020graph}, short cycles are identified and contracted; we adapt the criteria to \emph{resistance distances} in lieu of unweighted hop distances. The LRD threshold ensures each supernode aggregates edges that have sufficiently low $d_{eff}(p,q)$, while edges with higher $d_{eff}(p,q)$ remain across clusters and are re-inserted (or re-weighted) in the final sparsified graph $H$.

\noindent \emph{(3) Spectral approximation.}  
Standard spectral graph theory arguments \cite{spielman2011graph} show that removing or down-weighting edges of low effective resistance induces little change in $x^\top L_G x$ for all $x$. Conversely, preserving edges with large $d_{eff}(p,q)$ is crucial for maintaining the spectral signature of $L_G$. The result is a $(1\pm\varepsilon)$-approximation for sufficiently small $\varepsilon$.

Thus, LRD-based cycle decomposition extends short-cycle approaches to weighted graphs by anchoring cycle lengths in \emph{resistance} metrics. This achieves the final $(1\pm\varepsilon)$ spectral-approximation for $G$ in near-linear time.
\end{proof}

\subsection{Proof for the Relationship Between \texorpdfstring{$\gamma^F_{min}$}{gammaFmin} and \texorpdfstring{$\lambda_{\max}(L_X^+\,L_Y)$}{Lg}}
\label{proof_t1_min}

\begin{definition}[Minimum Distance Mapping Distortion]
Analogous to the definition of $\gamma^F_{max}$, we define
\[
\gamma^F_{min} \;\;=\;\;\min_{\substack{p,\,q\in V\\p\neq q}}
\,\frac{d_Y(p,q)}{\,d_X(p,q)\,}
\;=\;\min_{\substack{p,\,q\in V\\p\neq q}}
\,\frac{e_{p,q}^\top L_Y^+\,e_{p,q}}{\,e_{p,q}^\top L_X^+\,e_{p,q}\,},
\]
where \(L_X^+\) and \(L_Y^+\) are the Moore–Penrose pseudoinverses of the Laplacian matrices \(L_X\) and \(L_Y\), respectively. This quantity reflects the \emph{smallest} ratio of output distance to input distance, characterizing how close points in the \emph{output} manifold might originate from distant points in the \emph{input} manifold.
\end{definition}

\begin{proof}
Recall that
\[
\gamma^F_{min} 
\;=\;
\min_{\substack{p,\,q\in V\\p\neq q}}
\frac{e_{p,q}^\top L_Y^+\,e_{p,q}}{\,e_{p,q}^\top L_X^+\,e_{p,q}\,}.
\]
If we define \(v = e_{p,q}\) and restrict \(v\) such that \(v^\top \mathbf{1} = 0\) (\(\mathbf{1}\) being the all-ones vector, ensuring we stay within the subspace on which the Laplacian pseudoinverses are invertible), then
\[
\gamma^F_{min}
\;\;\ge\;\;
\min_{\substack{\|v\|\neq 0\\v^\top \mathbf{1}=0}}
\frac{v^\top L_Y^+\,v}{\,v^\top L_X^+\,v\,}.
\]
By the (min-)max version of the generalized Courant-Fischer theorem (applied to positive semidefinite matrices on the subspace orthogonal to \(\mathbf{1}\)), we have
\[
\min_{\substack{\|v\|\neq 0\\v^\top \mathbf{1}=0}}
\frac{v^\top L_Y^+\,v}{v^\top L_X^+\,v}
\;=\;
\lambda_{\min}\bigl(L_Y^+\,L_X\bigr).
\]
However, because \(L_Y^+L_X\) is invertible on that same subspace, one also obtains the relationship
\[
\lambda_{\min}\bigl(L_Y^+\,L_X\bigr)
\;=\;
\frac{1}{\,\lambda_{\max}\bigl(\,\bigl(L_Y^+\,L_X\bigr)^{-1}\bigr)}.
\]
Next, it can be shown that
\[
\bigl(L_Y^+\,L_X\bigr)^{-1} 
\;=\; 
L_X^+\,L_Y
\quad
\text{(on the subspace }v^\top \mathbf{1}=0\text{)},
\]
which gives
\[
\lambda_{\min}\bigl(L_Y^+\,L_X\bigr)
\;=\;
\frac{1}{\,\lambda_{\max}\bigl(L_X^+\,L_Y\bigr)}.
\]
Combining these steps, we conclude
\[
\gamma^F_{min}
\;\;\ge\;\;
\lambda_{\min}\bigl(L_Y^+\,L_X\bigr)
\;=\;
\frac{1}{\,\lambda_{\max}\bigl(L_X^+\,L_Y\bigr)}.
\]
This completes the proof.
\end{proof}

\subsection{SALMAN Score and Corresponding Proofs}
\label{sec:proof_inverse_cubic}

Specifically, we first compute the weighted eigensubspace matrix $V_r\in {\mathbb{R} ^{N\times r}}$ for spectral embedding on $G_X$ with $N$ nodes:
\vspace{-4pt}
\begin{equation}\label{subspace2}
V_r\overset{\mathrm{def}}{=}\left[ {v_1}{\sqrt {\lambda_1}},...,  {v_r}{\sqrt {\lambda_r }}\right],
\vspace{-4pt}
\end{equation}
where $\lambda_1, \lambda_2, ..., \lambda_r$ represent the first $r$ largest eigenvalues of $L_Y^+L_X$ and $v_1, v_2, ..., v_r$ are the corresponding eigenvectors. Let $u_1, u_2, ..., u_N$ denote the $N$  eigenvectors of $L_X L_Y^+$, respectively, while their corresponding eigenvalues are shared.  In addition, eigenvectors ${{u_i}}$ can  be constructed to satisfy:
\begin{equation}\label{formula_p-orth}
{u_i^\top L^+_X u_j^{}} = \left\{ \begin{array}{l}
1, ~~~i = j\\
0, ~~~i \ne j.
\end{array} \right.
\end{equation}
\begin{equation}\label{formula_p-orth2}
\Rightarrow{u_i^\top L^+_Y u_j^{}} = \left\{ \begin{array}{l}
\lambda_i, ~~~i = j\\
0, ~~~i \ne j.
\end{array} \right.
\end{equation}
Therefore, the following equations hold:
 \begin{equation}\label{eigen_dual}
 \begin{split}
 L_Y^+u_i&=\lambda_i L_X^+ u_i  \leftrightarrow L_Y^+L_X \left(L_Y^+u_i\right)=\lambda_i \left(L_Y^+u_i\right)\\ L_X v_i&=\lambda_i L_Y v_i \leftrightarrow L_Y^+L_X v_i=\lambda_i v_i 
\end{split}
\end{equation}
which leads to the following equation
\begin{equation}\label{eigen_dual2}
 \begin{split}
   v_i&=\beta_i L_Y^+ u_i\\ \Rightarrow u^\top_j v_i  &= \left\{ \begin{array}{l}
\beta_i \lambda_i, ~~~i = j\\
0, ~~~i \ne j.
\end{array} \right. 
\end{split}
\end{equation}
where $\beta_i$ denotes a scaling coefficient.
Without loss of generality,  $e_{p,q}$ can be expressed as a linear combination of  $u_i$ for $i=1,...,N$ as follows:
 \begin{equation}\label{formula_epq}
 e_{p,q}=\sum_{i=1}^N \alpha_i u_i.
\end{equation}
Then
$\gamma^F(p,q)$ can be rewritten as follows:
\begin{equation} \label{gammanew}
\begin{split}
\gamma^F(p,q)&=\frac{d_Y(p,q)}{d_X(p,q)}=\frac{e_{p,q}^\top L_Y^+ e_{p,q}}{e_{p,q}^\top L_X^+ e_{p,q}}\\
&=\frac{(\sum_{i=1}^N \alpha_i u_i)^\top L_Y^+(\sum_{i=1}^N \alpha_i u_i) }{(\sum_{i=1}^N \alpha_i u_i)^\top L_X^+(\sum_{i=1}^N \alpha_i u_i)}\\
&=\frac{\sum_{i=1}^N\sum_{j=1}^N \alpha_i\alpha_j u_i^\top L_Y^+u_j }{\sum_{i=1}^N\sum_{j=1}^N \alpha_i\alpha_j u_i^\top L_X^+u_j}\\
&=\frac{\sum_{i=1}^N \alpha^2_i u_i^\top L_Y^+u_i }{\sum_{i=1}^N \alpha^2_i  u_i^\top L_X^+u_i}\\
&=\frac{\sum_{i=1}^N\alpha^2_i\lambda_i}{\sum_{i=1}^N\alpha^2_i}.
\end{split}
\end{equation}

If the edge $({p,q})$  is  dominantly aligned with a single dominant generalized eigenvector $u_k$ where $1\le k \le r$, it  implies  $\forall i\neq k, \alpha_i \approx 0$ and thus $e_{p,q}\approx \alpha_k u_k$. Then:
\begin{equation} \label{gammanew2}
\begin{split}
\gamma^F(p,q)\approx \lambda_k.
\end{split}
\end{equation}
With $\bigl\|\,V_r^\top e_{p,q}\bigr\|_2^2$, We have:
\begin{equation}
\label{the_edge_weigt_1}
\begin{split}
    \|V_r^\top e_{p,q}\|_2^2&=\sum_{i=1}^r \lambda_i (v_i^\top e_{p,q})^2 \\
    &=\sum_{i=1}^r \lambda_i \left(\sum_{j=1}^N \alpha_j \beta_i u_j^\top L_Y^+ u_i \right)^2 \\
    &=\sum_{i=1}^r \alpha^2_i \beta^2_i\lambda^3_i  \\
    &\approx \alpha^2_k \beta^2_k \lambda^3_k \propto \left(\gamma^F(p,q)\right)^3
    \end{split}
\end{equation}

Consider $L_X^+\,L_Y$ whose eigenvalues we denote by $\mu_1,\ldots,\mu_N$ with corresponding eigenvectors $w_1,\ldots,w_N$. we then compute the weighted eigensubspace matrix $W_r\in {\mathbb{R} ^{N\times r}}$ for spectral embedding on $G_Y$ with $N$ nodes:
\vspace{-4pt}
\begin{equation}\label{subspace2}
W_r\overset{\mathrm{def}}{=}\left[ {w_1}{\sqrt {\mu_1}},...,  {w_r}{\sqrt {\mu_r }}\right],
\vspace{-4pt}
\end{equation}

Because \(L_X^+\,L_Y\) has eigenvalues \(\mu_i=1/\lambda_i\), and its eigenvectors \(w_i\) correspond in a reciprocal way, one obtains a parallel statement. In particular:
\begin{itemize}
\item If \(e_{p,q}\) aligns chiefly with the eigenvector \(w_k\) of \(L_X^+\,L_Y\) having eigenvalue \(\mu_k = 1/\lambda_k\),
\item Then \(\gamma^F(p,q)\approx \lambda_k\) as before.
\end{itemize}

A similar calculation to the Equation~\ref{the_edge_weigt_1} proof now yields
\[
  \bigl\|\,W_r^\top e_{p,q}\bigr\|_2^2
  \;=\;\sum_{i=1}^r \mu_i\,\bigl(w_i^\top e_{p,q}\bigr)^2
  \;\approx\;
  \text{const}\,\times (\mu_k)^{3}
  \;=\;
  \text{const}\,\times \Bigl(\tfrac{1}{\lambda_k}\Bigr)^3.
\]

Since \(\lambda_k\approx \gamma^F(p,q)\), we conclude
\[
\bigl\|\,W_r^\top e_{p,q}\bigr\|_2^2
\propto\;\gamma^F(p,q)^{-3}.
\]

Hence, $\bigl\|\,W_r^\top e_{p,q}\bigr\|_2^2 + \bigl\|\,V_r^\top e_{p,q}\bigr\|_2^2 \propto \gamma^F(p,q)^3+\gamma^F(p,q)^{-3}$.

\subsection{Experimental Setup}
\label{sec:appendix_exp_setup}

In this section, we provide details on the datasets, model configurations, training/finetuning protocols, and evaluation metrics used throughout our experiments. By clarifying each step, we ensure that our methodology is both \emph{transparent} and \emph{reproducible}.

\paragraph{Dataset.}
We evaluate on benchmark datasets such as \textsc{SST-2}, \textsc{MNLI}, \textsc{RTE}, \textsc{QNLI}, \textsc{QQP}, and \textsc{CoLA} to cover diverse classification objectives (sentiment analysis, natural language inference, and question classification). Each dataset is split into training, validation, and test sets following standard protocols (e.g., the GLUE benchmark~\citep{wang2018glue}). We tokenize inputs using the \emph{default} subword tokenizer for each model (e.g., BERT’s WordPiece or RoBERTa’s Byte-Pair Encoding), lowercasing as necessary. For SALMAN-Guided Attack experiment, we use AdvBench Harmful Behaviors dataset. JailBreak does not involve a training process, thus we did not split the dataset. We directly ranked the entire dataset of 520 data points.

\paragraph{Language Model.}
We evaluate on several benchmark language models such as BERT-base-uncased~\cite{devlin2018bert}, RoBERTa-base~\cite{liu2019roberta}, DistilBERT-base-uncased~\cite{sanh2019distilbert}, ALBERT-base-v2~\cite{lan2019albert}, GPT-2~\cite{radford2019language}, and LLaMA-7B-v2~\cite{touvron2023llama}.

\paragraph{Hyperparameter settings.}
We obtain \emph{Distance Mapping Distortion} scores for each sample by comparing input and output manifold distances (e.g., from \(\mathbf{z}_X\) to \(\mathbf{z}_Y\)). Summary of hyperparameters during DMD calculation is in Table~\ref{tab:hparams_summary}. To gauge how much finetuned models deviate from their pretrained checkpoints, we reference layer-wise similarity metrics such as CKA and STIR~\citep{neerudu2023robustness}.

\begin{table}[t]
\centering
\small
\caption{Summary of hyperparameters (\(k\) is for kNN graph construction and SPF is for our low-rresistance-diameter decomposition) used in our method for each \textit{(model, attack)} configuration. 
DIS refers to random selection of deletion, insertion, or swap.}
\label{tab:hparams_summary}
\begin{tabular}{l l c c}
\toprule
Model & Attack & \(k\) & \(SPF\) \\
\midrule
\multicolumn{4}{c}{\emph{SST-2}} \\
\midrule
BERT-base-uncased       & DIS        & 30  & 2 \\
RoBERTa-base            & DIS        & 30  & 2 \\
DistilBERT-base-uncased & DIS        & 30  & 2 \\
ALBERT-base-v2          & DIS        & 30  & 2 \\
GPT-2                   & DIS        & 10  & 2 \\
LLaMA-7B-v2             & DIS        & 10  & 2 \\
GPT-2                   & spaCy      & 20  & 2 \\
LLaMA-7B-v2             & spaCy      & 30  & 3 \\
GPT-2                   & TextAttack & 10  & 2 \\
LLaMA-7B-v2             & TextAttack & 10  & 2 \\
\midrule
\multicolumn{4}{c}{\emph{MNLI}} \\
\midrule
BERT-base-uncased       & DIS        & 30  & 2 \\
RoBERTa-base            & DIS        & 30  & 2 \\
DistilBERT-base-uncased & DIS        & 30  & 2 \\
ALBERT-base-v2          & DIS        & 30  & 2 \\
GPT-2                   & DIS        & 50  & 2 \\
LLaMA-7B-v2             & DIS        & 10  & 2 \\
GPT-2                   & spaCy      & 70 & 3 \\
LLaMA-7B-v2             & spaCy      & 70  & 2 \\
GPT-2                   & TextAttack & 20  & 3 \\
LLaMA-7B-v2             & TextAttack & 10  & 2 \\
\bottomrule
\end{tabular}
\end{table}

k-NN Ablation on SST-2 (GPT-2). 
We also evaluate the sensitivity of SALMAN to the choice of \(k\) in \(k\)-NN graph construction. Specifically, we vary \(k \in \{15, 20, 30\}\) for GPT-2 on SST-2 and compare the \emph{Kullback--Leibler Divergence} (KLD) and \emph{BERTScore} (Precision, Recall, \(F_{1}\)) for non-robust (NR) vs.\ robust (R) samples. As shown in Table~\ref{tab:knn_ablation}, increasing \(k\) does not drastically alter the distinction between robust and non-robust data; the non-robust subsets consistently exhibit higher KLD and slightly lower BERTScores, while robust subsets remain more stable under perturbations. This indicates that SALMAN is relatively insensitive to moderate changes in \(k\).

\begin{table}[h]
  \centering
  \caption{Effect of varying \(k\) in \(k\)-NN on robustness and similarity metrics (GPT-2, SST-2).}
  \label{tab:knn_ablation}
  \resizebox{\columnwidth}{!}{%
  \begin{tabular}{c c c c c c c c c}
    \toprule
    $k$ (kNN) & KLD (NR) & KLD (R) & Precision (NR) & Recall (NR) & $F_{1}$ (NR) & Precision (R) & Recall (R) & $F_{1}$ (R) \\
    \midrule
    15 & 0.1110 & 0.0000 & 0.9972 & 0.9978 & 0.9975 & 0.9990 & 0.9988 & 0.9989 \\
    20 & 0.0988 & 0.0003 & 0.9973 & 0.9978 & 0.9976 & 0.9993 & 0.9992 & 0.9992 \\
    30 & 0.1923 & 0.0000 & 0.9962 & 0.9970 & 0.9966 & 0.9992 & 0.9992 & 0.9992 \\
    \bottomrule
  \end{tabular}
  }
\end{table}

\subsection{Layer-wise STIR and CKA Results}
\label{sec:appendix_gpt2_sst2_rte}

In addition to the layer-wise comparison between normal and guided fine-tuning shown in Table~\ref{tab:gpt2_cola_small} (CoLA dataset), we replicate the same analysis for the \textbf{SST-2} and \textbf{RTE} tasks under GPT-2. Following the exact protocol of Section~\ref{subsec:guided-finetuning} and \citet{neerudu2023robustness}, we assign higher training weights to non-robust data (determined by our DMD ranking) and lower weights to robust data. As before, we measure:
\begin{itemize}
    \item \textbf{Validation Accuracy} on the downstream task,
    \item \textbf{STIR} (Similar Token Identity Representation) metrics \texttt{(m2m1, m1m2)} capturing how similar layer $i$ in the fine-tuned model $m2$ is to layer $j$ in the pre-trained model $m1$,
    \item \textbf{CKA} measuring layer-wise alignment between $m1$ and $m2$ embeddings.
\end{itemize}

\vspace{1em}
\noindent
\textbf{SST-2 Results.} Table~\ref{tab:gpt2_sst2_small} shows GPT-2’s layer-wise STIR and CKA under normal vs.\ guided fine-tuning on SST-2. Both approaches yield similar \emph{final accuracy} (0.9231 vs.\ 0.9232), yet the guided variant consistently achieves higher STIR/CKA scores in later layers. In particular, layer~12 sees a substantial jump in STIR(\texttt{m2m1}) from 0.0533 to 0.0867 and CKA from 0.1459 to 0.2039, indicating closer alignment to the pre-trained checkpoint.

\vspace{1em}
\noindent
\textbf{RTE Results.} In Table~\ref{tab:gpt2_rte_small}, we compare normal vs.\ guided fine-tuning for GPT-2 on the RTE dataset. While both runs converge similarly in accuracy (not shown here to save space), the guided approach again shows improved STIR and CKA alignment with the pre-trained checkpoint. For instance, layer~12 sees an increase from 0.2858 to 0.3393 in STIR(\texttt{m2m1}) and from 0.3458 to 0.3476 in CKA.

\begin{table}[t]
\centering
\small
\caption{Layer-wise STIR and CKA for GPT-2 on CoLA. 
Each cell shows the ``Normal fine-tuning / SALMAN-guided fine-tuning'', rounded to four decimal places. Normal fine-tuning validation accuracy is $0.7468$, SALMAN-guided fine-tuning validation accuracy is $0.7709$. $m1$ is the pre-trained model and $m2$ is the fine-tuned model. Better results are in \textbf{bold}.}
\label{tab:gpt2_cola_small}
% \resizebox{\columnwidth}{!}{%
\begin{tabular}{c c c c}
\toprule
Layer & STIR(m2m1) & STIR(m1m2) & CKA\\
\midrule
0  & 0.9623 / \textbf{0.9623} & 0.9623 / \textbf{0.9623} & 1.0000 / \textbf{1.0000}\\
1  & 0.9070 / \textbf{0.9073} & \textbf{0.9065} / 0.9053 & 0.9987 / \textbf{0.9987}\\
2  & 0.9688 / \textbf{0.9691} & 0.9690 / \textbf{0.9711} & \textbf{0.9936} / 0.9931\\
3  & 0.9848 / \textbf{0.9904} & \textbf{0.9678} / 0.9551 & \textbf{0.9856} / 0.9748\\
4  & 0.9853 / \textbf{0.9934} & 0.9690 / \textbf{0.9775} & 0.9836 / \textbf{0.9837}\\
5  & 0.9904 / \textbf{0.9928} & 0.9750 / \textbf{0.9752} & 0.9906 / \textbf{0.9909}\\
6  & 0.9897 / \textbf{0.9924} & 0.9697 / \textbf{0.9767} & 0.9920 / \textbf{0.9945}\\
7  & 0.9895 / \textbf{0.9927} & 0.9724 / \textbf{0.9833} & \textbf{0.9931} / 0.9909\\
8  & 0.9860 / \textbf{0.9914} & 0.9680 / \textbf{0.9854} & 0.9923 / \textbf{0.9936}\\
9  & 0.9825 / \textbf{0.9872} & 0.9666 / \textbf{0.9770} & 0.9905 / \textbf{0.9907}\\
10 & 0.9776 / \textbf{0.9833} & 0.9647 / \textbf{0.9762} & 0.9917 / \textbf{0.9928}\\
11 & 0.9730 / \textbf{0.9784} & 0.9628 / \textbf{0.9678} & 0.9893 / \textbf{0.9904}\\
12 & 0.4691 / \textbf{0.7233} & 0.6819 / \textbf{0.7924} & 0.5612 / \textbf{0.7251}\\
\bottomrule
\end{tabular}%
% }
\end{table}

\begin{table}[h!]
\centering
\small
\caption{Layer-wise STIR and CKA for GPT-2 on \textbf{RTE} under Normal vs.\ Guided fine-tuning. 
Each cell shows ``Normal Fine-tuning /  SALMAN-guided Fine-tuning'', rounded to four decimal places.
Better results in \textbf{bold}.}
\label{tab:gpt2_rte_small}
\begin{tabular}{c c c c}
\toprule
\textbf{Layer} & \textbf{STIR(m2m1)} & \textbf{STIR(m1m2)} & \textbf{CKA}\\
\midrule
0  & 0.9913 / \textbf{0.9913} & 0.9914 / \textbf{0.9914} & 1.0000 / \textbf{1.0000}\\
1  & 0.9786 / \textbf{0.9791} & 0.9776 / \textbf{0.9779} & \textbf{0.9986} / 0.9977\\
2  & \textbf{0.9859} / 0.9857 & \textbf{0.9859} / 0.9852 & 0.9976 / \textbf{0.9990}\\
3  & 0.9903 / \textbf{0.9920} & 0.9917 / \textbf{0.9918} & 0.9951 / \textbf{0.9987}\\
4  & 0.9897 / \textbf{0.9902} & 0.9897 / \textbf{0.9900} & 0.9885 / \textbf{0.9963}\\
5  & 0.9898 / \textbf{0.9908} & 0.9891 / \textbf{0.9916} & 0.9894 / \textbf{0.9981}\\
6  & 0.9865 / \textbf{0.9872} & 0.9869 / \textbf{0.9886} & 0.9908 / \textbf{0.9923}\\
7  & 0.9821 / \textbf{0.9829} & 0.9806 / \textbf{0.9839} & 0.9746 / \textbf{0.9801}\\
8  & 0.9758 / \textbf{0.9781} & 0.9709 / \textbf{0.9761} & 0.9407 / \textbf{0.9500}\\
9  & 0.9708 / \textbf{0.9724} & 0.9607 / \textbf{0.9696} & 0.9288 / \textbf{0.9492}\\
10 & 0.9564 / \textbf{0.9601} & 0.9359 / \textbf{0.9507} & 0.9028 / \textbf{0.9347}\\
11 & \textbf{0.9333} / 0.9331 & 0.9152 / \textbf{0.9265} & 0.9223 / \textbf{0.9390}\\
12 & 0.2858 / \textbf{0.3393} & 0.6131 / \textbf{0.6203} & 0.3458 / \textbf{0.3476}\\
\bottomrule
\end{tabular}
\end{table}

\begin{table}[h!]
\centering
\small
\caption{Layer-wise STIR and CKA for GPT-2 on \textbf{SST-2} under Normal vs.\ Guided fine-tuning.  
Each cell shows ``Normal Fine-tuning / Guided Fine-tuning'', rounded to four decimal places.  
\textbf{Acc} is the validation accuracy of each method. 
For STIR, $(\texttt{m2m1})$ compares the fine-tuned model $m2$ to the pre-trained model $m1$, and $(\texttt{m1m2})$ is the reverse; CKA measures embedding similarity. 
Better results are in \textbf{bold}.}
\label{tab:gpt2_sst2_small}
\begin{tabular}{c c c c}
\toprule
\multicolumn{4}{c}{\textbf{Validation Accuracy: } Normal = 0.9231,\quad  SALMAN-guided = 0.9232} \\
\midrule
\textbf{Layer} & \textbf{STIR(m2m1)} & \textbf{STIR(m1m2)} & \textbf{CKA}\\
\midrule
0  & 0.9913 / \textbf{0.9913} & 0.9912 / \textbf{0.9912} & 1.0000 / \textbf{1.0000}\\
1  & 0.9763 / \textbf{0.9771} & 0.9784 / \textbf{0.9787} & 0.9963 / \textbf{0.9974}\\
2  & 0.9784 / \textbf{0.9789} & 0.9762 / \textbf{0.9767} & 0.9971 / \textbf{0.9973}\\
3  & 0.9703 / \textbf{0.9713} & 0.9366 / \textbf{0.9414} & 0.9460 / \textbf{0.9469}\\
4  & 0.9661 / \textbf{0.9715} & 0.9549 / \textbf{0.9608} & 0.9773 / \textbf{0.9800}\\
5  & 0.9736 / \textbf{0.9757} & 0.9469 / \textbf{0.9589} & 0.9705 / \textbf{0.9738}\\
6  & 0.9649 / \textbf{0.9704} & 0.9343 / \textbf{0.9450} & 0.9568 / \textbf{0.9604}\\
7  & 0.9618 / \textbf{0.9672} & 0.9389 / \textbf{0.9476} & 0.9642 / \textbf{0.9675}\\
8  & 0.9663 / \textbf{0.9703} & 0.9514 / \textbf{0.9598} & 0.9800 / \textbf{0.9825}\\
9  & 0.9435 / \textbf{0.9553} & 0.9473 / \textbf{0.9545} & 0.9717 / \textbf{0.9787}\\
10 & 0.9230 / \textbf{0.9504} & 0.9486 / \textbf{0.9573} & 0.9599 / \textbf{0.9774}\\
11 & 0.8567 / \textbf{0.9208} & 0.9328 / \textbf{0.9426} & 0.9166 / \textbf{0.9562}\\
12 & 0.0533 / \textbf{0.0867} & 0.7504 / \textbf{0.7755} & 0.1459 / \textbf{0.2039}\\
\bottomrule
\end{tabular}
\end{table}

\vspace{1em}
\noindent
\textbf{Discussion.}
Similar to our observations on \textsc{CoLA} (Table~\ref{tab:gpt2_cola_small}), placing higher emphasis on non-robust data (i.e., higher DMD samples) preserves downstream performance while bringing the fine-tuned layers closer to the original pre-trained representations. These improvements in STIR and CKA suggest \emph{reduced representational drift}, consistent with the intuition that focusing on “hard” samples forces the model to retain more generalizable features from pre-training~\citep{cheng2021spade,zhu2023improving}.

Overall, these extended results on SST-2 and RTE corroborate our main findings: \emph{robustness-guided fine-tuning} effectively balances task performance with better alignment to the pre-trained checkpoint across multiple datasets.

\subsection{Weighted Fine-Tuning and Integration with ROSE}
\label{sec:appendix_weighting}

\paragraph{Motivation: Focus on Non-Robust Data.}
As discussed in Section~\ref{subsec:guided-finetuning}, prior studies have shown that directing more attention to non-robust (``hard'') samples during training can improve model generalizability and resilience~\citep{cheng2021spade,zhu2023improving}. Our approach identifies these difficult samples via the  SALMAN-based ranking and then \emph{assigns higher training weights} to them, while simultaneously down-weighting samples that appear robust. We follow the finetuning protocol of \citet{neerudu2023robustness}, hypothesizing that emphasizing harder samples preserves more of the pre-trained model’s versatility. This reduces the risk of overfitting to “easy” data and yields representations closer to the original checkpoint (see STIR/CKA results in Appendix~\ref{sec:appendix_gpt2_sst2_rte}).

\subsubsection{Weighting Schemes for Guided Fine-Tuning}
\textbf{Linear Schedule.}  
We sort all training samples in descending order of their DMD values (highest DMD = most non-robust), then map each sample to a weight $w\in[0,1]$ proportional to its position in this ranking. Concretely, if the highest-DMD sample is indexed as rank $0$, it receives weight \(\approx1.0\), whereas the lowest-DMD sample (rank $n{-}1$) receives weight near $0.0$. Intermediate samples smoothly interpolate between these extremes.

\paragraph{Combining with SOTA Robust Training (ROSE).}
We further integrate our DMD-based weighting into \textbf{ROSE: Robust Selective Fine-tuning}~\citep{jiang2022rose}, which filters out spurious parameter updates by comparing dropout-induced distributions at each iteration:
\[
L_{\mathrm{KL}}^{(t)} \;=\; D_{\mathrm{KL}}(P_t\;\|\;P'_t) \;+\; D_{\mathrm{KL}}(P'_t\;\|\;P_t).
\]
ROSE removes parameter changes that inflate $L_{\mathrm{KL}}$ excessively, thus improving adversarial resilience.

\textbf{Per-sample Weight $w(x)$ for Joint Optimization}
Alternatively, we employ a logistic transition-based function partitioned into intervals:
\begin{itemize}
    \item \textbf{Top-25\% Non-Robust} (i.e., highest DMD) can receive a \emph{larger} weight (e.g., $2.0$),
    \item \textbf{Middle-Range} samples gradually decrease from $1.0$ to $0.0$ in stepwise logistic transitions,
    \item \textbf{Bottom-5\% Most Robust} eventually gets weight $0.0$ (or near-zero).
\end{itemize}
This piecewise approach allows a \emph{finer distinction} between very hard, moderately challenging, and trivially easy samples.

\noindent
\textbf{Joint Optimization.}  
We incorporate our per-sample weight $w(x)$ into ROSE’s fine-tuning loss. Specifically, if the original ROSE objective is
\[
\mathcal{L}_{\mathrm{ROSE}}(\theta_t) 
\;=\;
\mathbb{E}_{x \sim \mathcal{D}}\bigl[L_{\mathrm{task}}(x, \theta_t) + \lambda L_{\mathrm{KL}}^{(t)}\bigr],
\]
then our \emph{combined} objective is
\[
\mathcal{L}_{\mathrm{ROSE+Guided}}(\theta_t) 
\;=\;
\mathbb{E}_{x \sim \mathcal{D}}\Bigl[
\,w(x)\,\cdot L_{\mathrm{task}}(x, \theta_t)
+\;\lambda\,L_{\mathrm{KL}}^{(t)}\Bigr].
\]
Hence, the model is “selective” not only at the \emph{parameter} level (via $L_{\mathrm{KL}}$) but also at the \emph{sample} level (via DMD-based weighting).

\subsection{From PGM to Manifold: Validating on Graph Benchmarks}
\label{sec:appendix_pgm_manifold}

Although our primary interest is applying the PGM-based manifold to NLP data (where nodes represent text embeddings), we first validate how well our spectral sparsification and resistance distance preservation works on \emph{standard graph benchmarks}, namely \textbf{Cora}, \textbf{Citeseer}, and \textbf{PolBlogs}. These datasets are widely used in the GNN literature and offer:
\begin{itemize}
    \item \textbf{Well-defined adjacency}: Each graph provides a clear baseline for measuring changes in effective resistance.
    \item \textbf{Known benchmarks for graph-based algorithms}: This allows direct comparison of spectral or manifold-like approaches without the additional complexity of NLP text embedding.
\end{itemize}
In other words, while our ultimate goal is to build a \emph{manifold} for robustness analysis in transformer-based language models, these classic graph datasets serve as an \emph{intermediary check} to confirm that the PGM manifold indeed preserves \emph{resistance distances} in large-scale graphs.

\vspace{5pt}
\noindent
\textbf{Why Graph Benchmarks Instead of NLP Data?}
\begin{itemize}
    \item \emph{Ground-Truth Adjacency}:
    For cora/citeseer/polblogs, the adjacency matrix is explicitly available, enabling a direct before/after comparison of edge sparsity and distance correlation. In contrast, NLP data initially lacks a clear “graph,” so we must approximate edges (e.g., via $k$-NN). Verifying the correctness of our approach on well-studied graph datasets ensures that the spectral sparsification steps \emph{properly} preserve distances.
    \item \emph{Easier Resistance Verification}:
    By default, each node in these graph benchmarks is associated with a known set of neighbors. We can compute full-pairwise effective resistance or measure Pearson, Spearman, and MSE between original and sparsified graphs (Table~\ref{tab:graph_spf_summary}). This level of straightforward measurement is less trivial in NLP tasks, where adjacency depends on embedding similarity.
\end{itemize}

\noindent
\textbf{Experiment Setup.}
\begin{enumerate}
    \item \textbf{Compute original resistance distances} for each pair of nodes in the unsparsified graph.
    \item \textbf{Apply our SPF (Spectral Pruning via effective-resistance)} procedure at various parameters (e.g., $\texttt{param} \in \{2,3,4\}$), generating a pruned graph that discards edges with smaller distance ratios.
    \item \textbf{Quantify distance preservation} via Pearson correlation, Spearman correlation, MSE, and relative error (\texttt{RelErr}) between the original and the pruned graph’s resistance distances.
    \item \textbf{Measure final edge count} as a fraction of the original adjacency size.
\end{enumerate}

\begin{table}[t]
\centering
\caption{SPF results on three datasets (Cora, Citeseer, Polblogs). 
For each dataset, we vary the SPF parameter in \{2, 3, 4\}, 
then measure how well the transformed adjacency preserves the original resistance distances 
(Pearson / Spearman correlation, MSE, relative error). 
``Edges\%'' indicates the proportion of edges retained relative to the original graph.}
\label{tab:graph_spf_summary}
\begin{tabular}{l c c c c c c}
\toprule
\textbf{Dataset} & \textbf{SPF} & \textbf{Pearson} & \textbf{Spearman} & \textbf{MSE} & \textbf{RelErr} & \textbf{Edges\%}\\
\midrule
\textbf{cora}     & 2 & 0.9029 & 0.8899 & 0.58045 & 0.3511 & 80.29\% \\
\textbf{cora}     & 3 & 0.8602 & 0.8495 & 1.01185 & 0.5178 & 74.51\% \\
\textbf{cora}     & 4 & 0.8113 & 0.7988 & 1.76080 & 0.7074 & 70.21\% \\
\midrule
\textbf{citeseer} & 2 & 0.9475 & 0.9475 & 0.89848 & 0.2463 & 80.48\% \\
\textbf{citeseer} & 3 & 0.9220 & 0.9190 & 1.67925 & 0.3658 & 75.71\% \\
\textbf{citeseer} & 4 & 0.9074 & 0.9014 & 2.46463 & 0.4674 & 72.25\% \\
\midrule
\textbf{polblogs} & 2 & 0.9565 & 0.9693 & 0.02916 & 0.3209 & 67.58\% \\
\textbf{polblogs} & 3 & 0.9090 & 0.9356 & 0.07819 & 0.6778 & 53.19\% \\
\textbf{polblogs} & 4 & 0.8323 & 0.8696 & 0.20026 & 1.4342 & 37.03\% \\
\bottomrule
\end{tabular}
\end{table}

\vspace{5pt}
\noindent
\textbf{Results and Analysis.}
Table~\ref{tab:graph_spf_summary} summarizes the outcomes on \textbf{Cora}, \textbf{Citeseer}, and \textbf{PolBlogs}. For each dataset:
\begin{itemize}
    \item \textbf{Pearson \& Spearman correlation} remain high ($>0.80$) even when we prune roughly $20\text{-}40\%$ of the edges, confirming that the principal global and local distance structures remain intact.
    \item \textbf{MSE and RelErr} naturally increase with more aggressive pruning, yet remain within acceptable ranges for many use-cases (e.g., GNN training, manifold-based clustering).
    \item \textbf{Sparsification Rate} (\texttt{Edges\%}) indicates that by increasing the SPF parameter, we can achieve increasingly compact graphs without catastrophically degrading the resistance-distance correlation.
\end{itemize}
In short, these results validate that our spectral-pruning approach effectively maintains key \emph{manifold properties} (represented by resistance distances) across standard graph benchmarks. By extension, we expect similar fidelity in large-scale NLP tasks once we construct an initial $k$-NN or adjacency graph from text embeddings.

Having verified the correctness of our PGM manifold construction on well-known graph datasets, we now apply the same principles (near-linear spectral sparsification plus Laplacian-based $\Theta$ construction) to build manifolds for high-dimensional text embeddings. This ensures that the subsequent distance analyses in our transformer robustness framework rely on an \emph{accurate} and \emph{scalable} manifold, preserving essential local and global distances just as effectively as in these classic graph scenarios.

\subsection{Effective Resistance Distance}
\label{sec:appendix_eff_resistance}

\paragraph{Motivation and Intuition.}
In graph-based methods, the \emph{effective resistance distance} (also called \emph{resistance distance} in electrical-network parlance) provides a powerful metric for understanding the relationship between pairs of nodes. Unlike simple shortest-path lengths, effective resistance captures both local and global connectivity: if two nodes are connected by many parallel paths, they have lower effective resistance than nodes primarily joined by a single, bottleneck path \citep{spielman2011spectral}.

\paragraph{Electrical Network Interpretation.}
One way to grasp effective resistance is to imagine placing a 1-Ohm resistor on each edge of the graph and then viewing the entire graph as an electrical circuit:
\begin{itemize}
    \item Inject 1 amp of current into node $u$ and extract it from node $v$.
    \item Let $\varphi(x)$ be the resulting electrical potential at any node $x$ in the network.
    \item The \emph{effective resistance distance} $R_{\mathrm{eff}}(u,v)$ is then \emph{the voltage difference} between $u$ and $v$, i.e., $\varphi(u)-\varphi(v)$, required to sustain that 1-amp current.
\end{itemize}
Thus, if there are many alternative routes (parallel edges) from $u$ to $v$, the network offers “lower resistance” between them, indicating $u$ and $v$ are closely tied in the graph’s connectivity structure \citep{chandra1989electrical, ellens2011effective}.

\paragraph{Mathematical Formulation via Laplacian Pseudoinverse.}
Let $G=(V,E)$ be an undirected, connected graph with $n=|V|$ nodes. 
Denote its \emph{Laplacian matrix} by $L_G = D - W$, where $D$ is the diagonal degree matrix and $W$ is the adjacency (or edge-weight) matrix. 
Since $L_G$ is positive semidefinite and has rank $n-1$ for a connected graph, it admits a Moore-Penrose pseudoinverse $L_G^+$ \citep{mohar2004graph,spielman2011spectral}. 
For nodes $p$ and $q$:
\[
\textstyle
R_{\mathrm{eff}}(p,q) 
\;=\;
(e_p - e_q)^\top \,L_G^+\, (e_p - e_q),
\]
where $e_p$ is the standard basis vector (all zeros except a 1 in the $p$-th coordinate). Intuitively, $L_G^+$ encodes global connectivity, so $R_{\mathrm{eff}}(p,q)$ measures “how difficult it is to flow current” from $p$ to $q$ across $G$ \citep{babic2002resistance}.

\paragraph{Example: Line Graph vs.\ Square Graph.}
To illustrate how the \emph{effective resistance} distance can differ substantially from the naive (hop-count) distance, consider:

\begin{itemize}
    \item \textbf{Line Graph with 3 Nodes} \(\{1,2,3\}\) and unit-weight edges \((1,2)\) and \((2,3)\). The hop distance from node 1 to node 3 is \(2\). When modeled as a resistor network, each edge contributes 1 ohm in series; thus, the effective resistance between node 1 and node 3 is 
    \[
    R_{\mathrm{eff}}(1,3) \;=\; 1 + 1 \;=\; 2.
    \]
    \item \textbf{Square Graph with 4 Nodes} \(\{1,2,3,4\}\) and edges \((1,2)\), \((2,3)\), \((3,4)\), \((4,1)\), each of unit weight. 
    The naive (hop) distance from node 1 to node 3 is \(2\) (e.g., via \(1 \to 2 \to 3\) or \(1 \to 4 \to 3\)). However, in the resistor-network view, there are two distinct 2-edge paths running in \emph{parallel} between node 1 and node 3:
    \[
        1\!\to\!2\!\to\!3
        \quad\text{and}\quad
        1\!\to\!4\!\to\!3.
    \]
    Each path alone would have resistance \(1+1 = 2\). Because they are in parallel, the total effective resistance is 
    \[
    R_{\mathrm{eff}}(1,3) \;=\; \Bigl(\frac{1}{2} + \frac{1}{2}\Bigr)^{-1} \;=\; 1.
    \]
\end{itemize}

\begin{figure}[ht]
\centering
\includegraphics[width=0.6\textwidth]{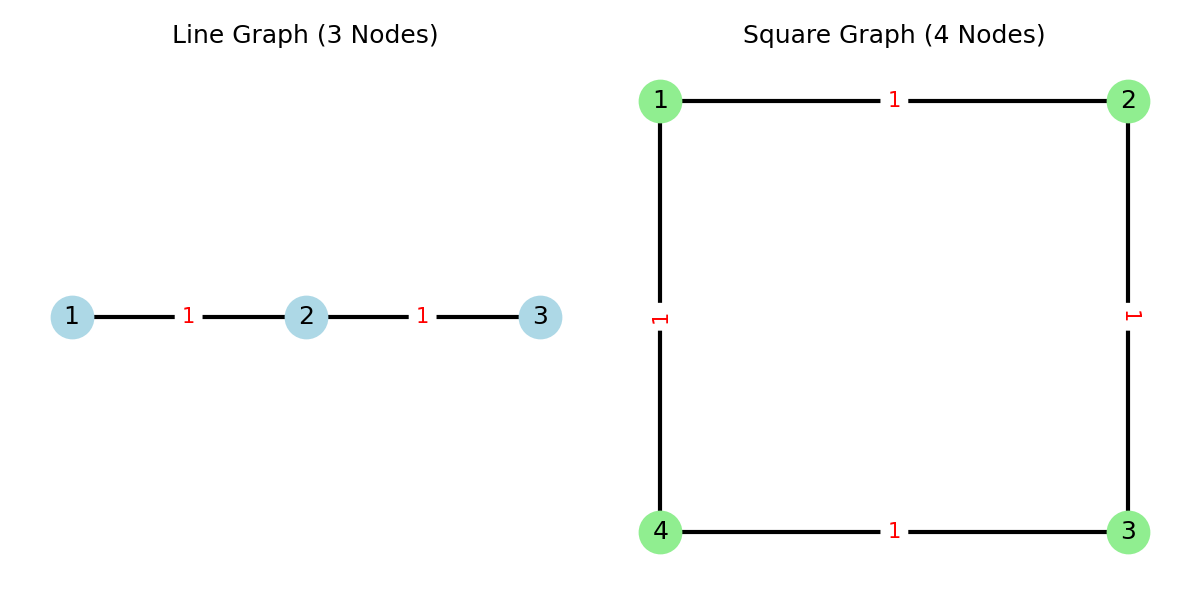}
\caption{\textbf{Line vs.\ Square Graph Examples.} 
\emph{(Left)} The line graph with nodes \(\{1,2,3\}\). 
\emph{(Right)} The square graph with nodes \(\{1,2,3,4\}\). 
Although both have a hop distance of 2 between node 1 and node 3, 
the \emph{effective resistance} differs significantly: 
it is \(R_{\mathrm{eff}}(1,3)=2\) in the line graph (two edges in series), 
versus \(R_{\mathrm{eff}}(1,3)=1\) in the square graph 
(two parallel 2-edge paths).}
\label{fig:line_square_comparison}
\end{figure}

These simple examples illustrate that the effective resistance distance may diverge from the naive, purely local distance. For node pairs in a graph with parallel paths, the effective resistance is often smaller than the hop count would suggest. By contrast, if all paths between two nodes lie strictly in series (as in a line graph), the effective resistance grows as a sum of edge resistances. Such distinctions are at the heart of why resistance-based metrics can better capture global connectivity and structural nuances in graph-based manifold analysis.

\paragraph{Relevance to Robustness and Graph Learning.}
The notion of effective resistance has become increasingly relevant for:
\begin{itemize}
    \item \textbf{Spectral Graph Sparsification}: 
    Low- and high-resistance edges are treated differently; small-resistance edges indicate redundancy, enabling fast approximation algorithms \citep{spielman2011graph, spielman2011spectral}.
    \item \textbf{Commute/Random Walk Times}:
    $R_{\mathrm{eff}}(p,q)$ also relates to the expected commute time of a random walk between $p$ and $q$ \citep{chandra1989electrical}, linking local connectivity to global diffusion properties.
    \item \textbf{Manifold Preserving Embeddings}:
    By preserving effective resistance distances, one can maintain both local neighborhoods and global circuit-like structure in a final embedding or graph model \citep{ellens2011effective, feng2021sgl}.
\end{itemize}
In short, effective resistance unifies local and global connectivity aspects, making it ideal for measuring how perturbations might propagate through a network—and by extension, how to keep the manifold structure stable in large-scale data (e.g., NLP embeddings).

\subsection{Empirical Evidence of \texorpdfstring{$(\gamma^F_{\min})^{-1}$}{Lg} Capturing “Collapses”}
\label{sec:appendix_bilip_empirical}

In Section~\ref{subsec:dmd-calculation}, we highlighted how a large $(\gamma^F_{\min})^{-1}$ indicates another dimension of fragility: \emph{distant} inputs becoming overly close in the output space. Below, we provide empirical results on multiple model--task combinations, measuring:
\begin{itemize}
    \item \textbf{Cosine Similarity (Cos)} between original vs.\ perturbed embeddings, 
    \item \textbf{KL Divergence (KLD)} between output distributions,
    \item for both \textit{non-robust} vs.\ \textit{robust} samples, under A: $\gamma^F_{\max}$ or B: $\gamma^F_{\max}+(\gamma^F_{\min})^{-1}$ setting.
\end{itemize}
A significant gap in Cos or KLD between robust and non-robust samples suggests the model \emph{amplifies} small differences in the non-robust subset (or “collapses” large differences). Conversely, if robust samples remain stable, it aligns with a lower distortion (or higher $\gamma^F_{\min}$).

\begin{table}[h]
\centering
\small
\caption{Comparisons of Cosine Similarity (\texttt{Cos}) and KL Divergence (\texttt{KLD}) across \textit{non-robust} vs.\ \textit{robust} subsets, under \textbf{A: $\gamma^F_{\max}$ or B: $\gamma^F_{\max}+(\gamma^F_{\min})^{-1}$} setting. Selected samples are attacked by spaCy. Each row shows: (1)~model+dataset, (2)(3)~Non-robust Cos, (4)(5)~Robust Cos, (6)(7)~Non-robust KLD, (8)(9)~Robust KLD. Higher Cos / lower KLD typically indicates more stable behavior. Better results are in \textbf{bold}.}
\label{tab:dmd_bilip_empirical}
\begin{tabular}{lcccccccc}
\toprule
\textbf{Model + Task} & \multicolumn{2}{c}{\textbf{Non-rob Cos}} & \multicolumn{2}{c}{\textbf{Rob Cos}} & \multicolumn{2}{c}{\textbf{Non-rob KLD}} & \multicolumn{2}{c}{\textbf{Rob KLD}}\\
\cmidrule(lr){2-3}\cmidrule(lr){4-5}\cmidrule(lr){6-7}\cmidrule(lr){8-9}
 & A & B & A & B & A & B & A & B \\
\midrule
\textbf{BERT, RTE}  & 0.9194 & \textbf{0.9091} & 0.9282 & \textbf{0.9407} & 0.00794 & \textbf{0.00884} & 0.00709 & \textbf{0.00605} \\
\textbf{BERT, SST-2} & 0.9368 & \textbf{0.9358} & 0.9968 & \textbf{0.9969} & 0.00631 & \textbf{0.00647} & 0.00033 & \textbf{0.00032} \\
\textbf{GPT-2, RTE} & 0.9755 & \textbf{0.9662} & 0.9844 & \textbf{0.9917} & 0.01992 & \textbf{0.01992} & 1.14e-13 & \textbf{9.44e-14} \\
\textbf{GPT-2, SST-2} & 0.9730 & \textbf{0.9634} & 0.9989 & \textbf{0.9988} & 0.12331 & \textbf{0.15453} & 2.21e-06 & \textbf{4.38e-07} \\
\textbf{LLaMA-7Bv2, RTE} & 0.9511 & \textbf{0.9438} & 0.9537 & \textbf{0.9582} & 0.6998 & \textbf{0.7733} & 0.6764 & \textbf{0.4797} \\
\textbf{LLaMA-7Bv2, SST-2} & \textbf{0.9490} & 0.9491 & 0.9777 & \textbf{0.9779} & \textbf{0.53032} & 0.52974 & 0.21646 & \textbf{0.17981} \\
\bottomrule
\end{tabular}
\end{table}

\vspace{1em}
\noindent
\textbf{Observations.}

Empirically, when we \emph{combine} both $\gamma^F_{\max}$ and $(\gamma^F_{\min})^{-1}$ (e.g., by ranking samples via $\gamma^F_{\max} + (\gamma^F_{\min})^{-1}$), we obtain a more accurate partition of robust vs.\ non-robust data than using $\gamma^F_{\max}$ alone. 
Specifically:
\begin{itemize}
    \item \textbf{Robust subset} selected by $\bigl[\gamma^F_{\max} + (\gamma^F_{\min})^{-1}\bigr]$ displays \emph{higher} cosine similarity and \emph{lower} KLD relative to a purely $\gamma^F_{\max}$-based choice, 
    \item \textbf{Non-robust subset} exhibits \emph{lower} cosine similarity and \emph{higher} KLD, indicating stronger local instability.
\end{itemize}
This confirms that jointly considering \emph{expansions} ($\gamma^F_{\max}$) and \emph{collapses} ($(\gamma^F_{\min})^{-1}$) provides a more fine-grained characterization of model robustness—reinforcing the notion that both extremes of the distortion spectrum matter for local manifold analysis.

\subsection{Scalability and Efficiency}
\label{subsec:scalability_efficiency}

Table~\ref{tab:efficiency_runtimes} reports total wall-clock time (in seconds) for \emph{embedding the dataset}, \emph{constructing the manifold graph}, and \emph{computing DMD} on standard hardware. Notably, even the largest GLUE tasks remain tractable. For instance, MNLI (393k samples) takes $\approx 6060$ seconds ($\sim1.7$ hours), which is a one-time cost. Smaller tasks like QNLI (105k) finish in $\sim12$ minutes. These results underscore that SALMAN is viable for mainstream NLP benchmarks. For extremely large datasets, approximate or distributed strategies can be employed for further scalability.

\begin{table}[!h]
\centering
\caption{{SALMAN runtime across different GLUE tasks. Approx.\ sample counts and total runtime on typical hardware.}}
\label{tab:efficiency_runtimes}
\begin{tabular}{lcc}
\toprule
\textbf{Dataset} & \textbf{\#Samples} & \textbf{Runtime (sec)} \\
\midrule
SST-2   & $\sim67$k  & 642.4  \\
RTE     & $\sim2.5$k & 12.0   \\
QNLI    & $\sim105$k & 736.3  \\
MNLI    & $\sim393$k & 6060.2 \\
\bottomrule
\end{tabular}
\end{table}

Thus, while SALMAN does require a modest upfront cost to build the manifold and compute distortions, the resulting robustness ranking can be reused for downstream tasks (e.g., adversarial evaluation, fine-tuning). This amortizes the cost and keeps the approach practical for modern NLP pipelines.

\subsection{More Attack Experiment Results}
\label{appendix:attack}
 Figure~\ref{fig:decile_asr_plot} (left) shows that ASR is highest for the first decile (most non-robust) and consistently decreases as samples become more robust in higher deciles. This confirms that SALMAN ranking provides a reliable gradient for identifying vulnerable data points.

\begin{figure}[!h]
    \centering
    \includegraphics[width=0.48\linewidth]{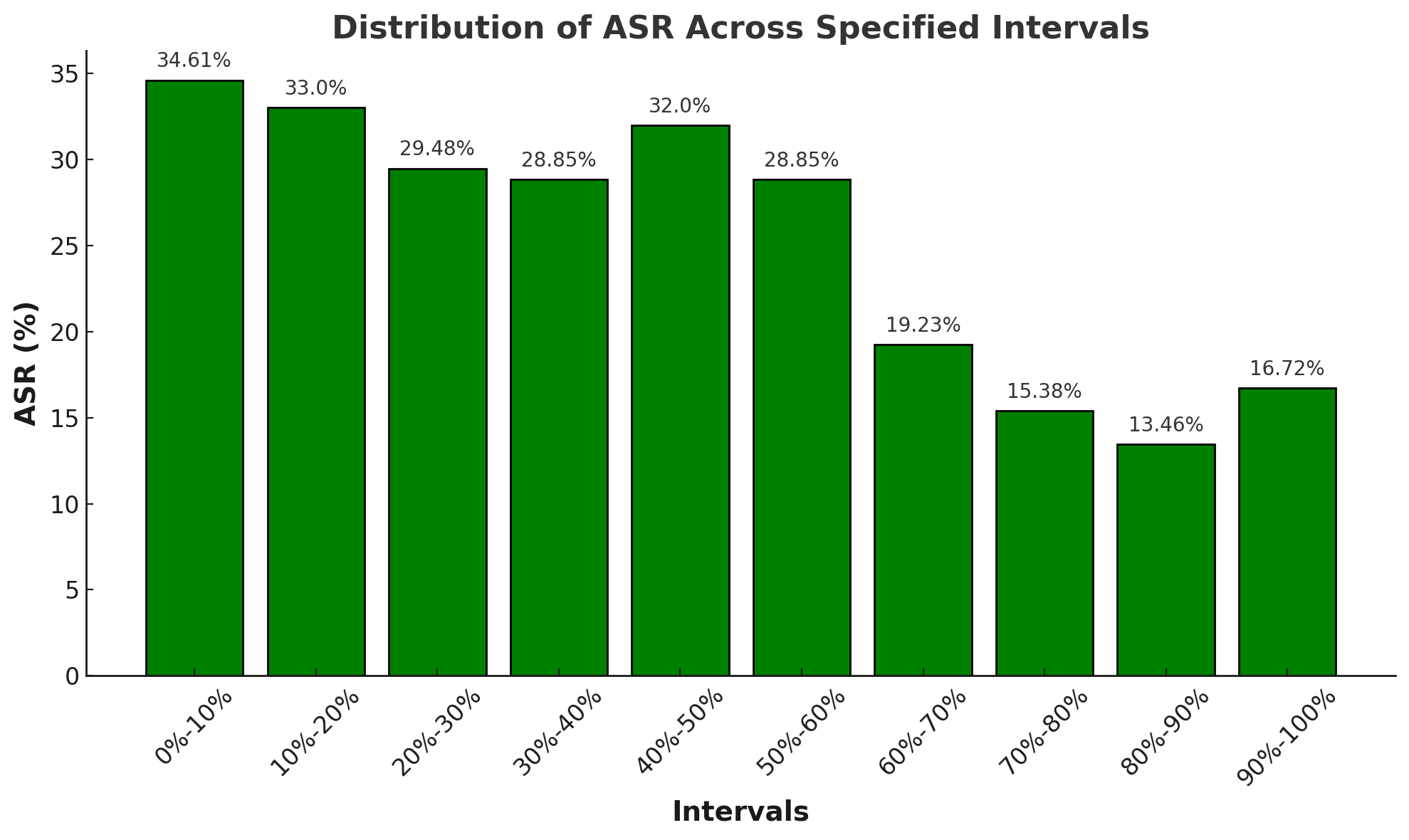}~~
    \includegraphics[width=0.48\linewidth]{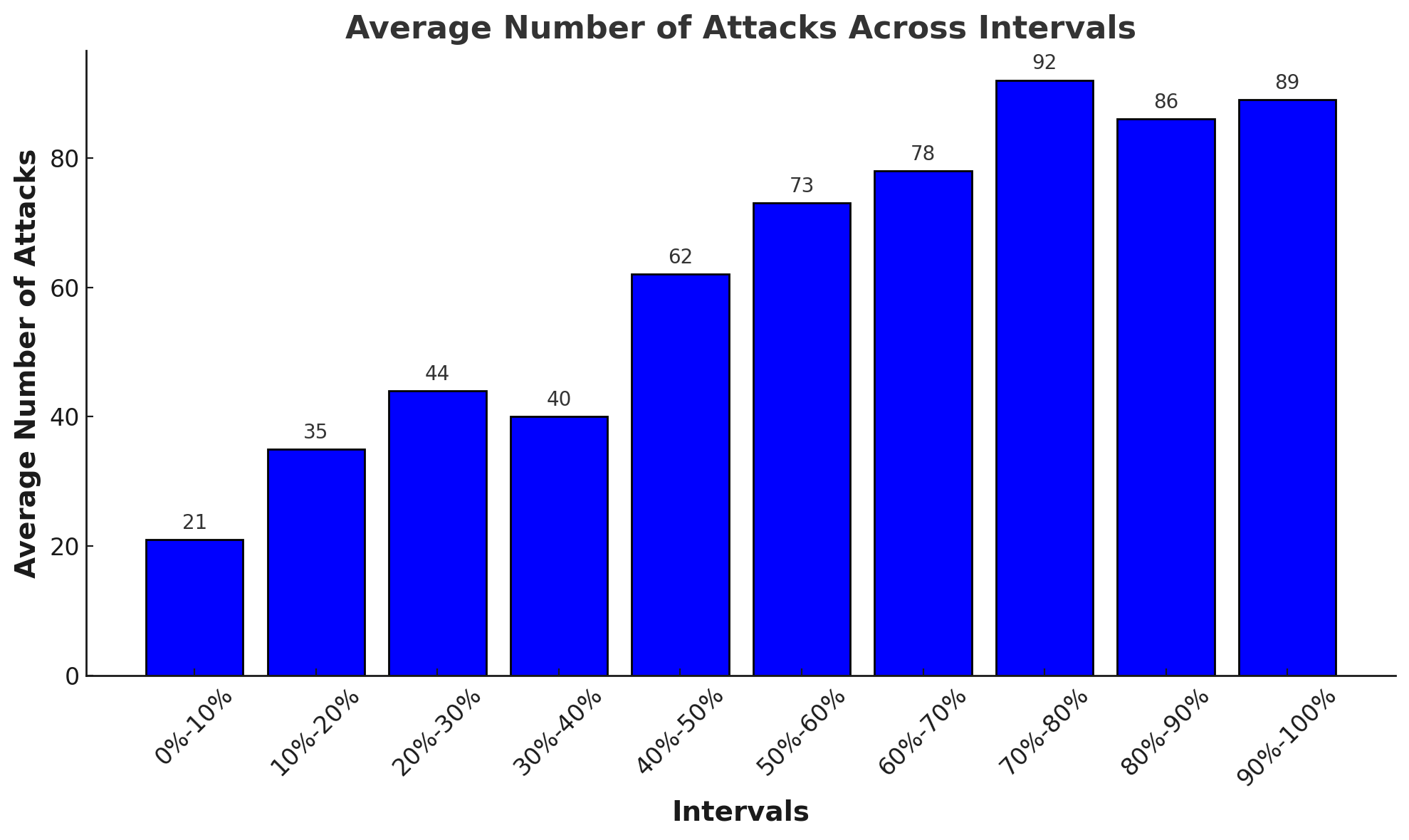}
    \caption{\emph{(Left)}  Attack Success Rate (ASR) across deciles of non-robustness. \emph{(Right)} Average number of AutoDAN steps needed for successful attack on robust vs.\ non-robust subsets. Non-robust samples require fewer steps, highlighting their vulnerability.}
    \label{fig:decile_asr_plot}
\end{figure}

% \subsubsection{Measuring Attack Steps}
We further follow GCG~\cite{zou2023universal} and AutoDAN~\cite{liu2023autodan} to measure the average number of attack steps required. By default, GCG uses a fixed 250 steps for each trial, but we adapt the AutoDAN approach to run up to 100 steps. Figure~\ref{fig:decile_asr_plot} (right) shows that non-robust samples require \emph{significantly fewer} steps for successful attack, whereas robust samples demand more queries to break. This corroborates our SALMAN-based ranking.

% \subsubsection{Proxy-Based SALMAN Ranking}
One may wonder if SALMAN must be computed on the exact same model we later attack. We investigate using GPT-2, LLaMA2-7B, or LLaMA3-8B embeddings as a “proxy” for SALMAN ranking, then testing the transferability of the attack to the target LLM. Table~\ref{tab:proxy_salman} shows that the Attack Success Rate (ASR) remains quite similar across each proxy’s ranking, suggesting that SALMAN is fairly robust to model variations.

\end{document}